\documentclass[11pt]{elsarticle}
\usepackage{geometry}
\usepackage{hanet}
\usepackage{url}
\usepackage{tikz}
\usepackage{amsfonts}
\usepackage{mathrsfs}
\usepackage{amssymb}
\usepackage{amsmath}
\usepackage{amsthm}
\usepackage{dsfont}
\usepackage{multirow}
\usepackage{booktabs}
\usepackage{tabularx}
\usepackage{graphicx}
\usepackage[linesnumbered,lined,boxed,commentsnumbered,ruled]{algorithm2e}
\usepackage{microtype}

\newcommand{\R}{\mathbb{R}}
\newcommand{\gph}{\mathcal{G}} % graph G

\newcommand{\eigfm}[1][\ell]{\phi_{#1}}
\newcommand{\ws}[1][j]{\mathcal{S}^{(#1)}} % weighted sum for ADFT
\newcommand{\vj}[1][j]{v^{(#1)}} % kth vertices of level-j graph
\newcommand{\wj}[1][j]{w^{(#1)}} % weight for kth vertices of level-j graph

\usepackage{xcolor}

\usepackage{hyperref}

\makeatletter
\def\ps@pprintTitle{%
 \let\@oddhead\@empty
 \let\@evenhead\@empty
 \def\@oddfoot{}%
 \let\@evenfoot\@oddfoot}
\makeatother

\begin{document}

\begin{frontmatter}

\title{Fast Haar Transforms for Graph Neural Networks\tnoteref{tdate}}
\tnotetext[tdate]{\today}
%\tnotetext[tifn]{The last author was supported under the Australian Research Council's \emph{Discovery Project} DP160101366.}
\author[addrzjnu,addrunsw]{Ming Li}
\ead{ming.li.ltu@gmail.com}

\author[addrprin]{Zheng Ma}
\ead{zhengm@princeton.edu}

\author[addrunsw]{Yu Guang Wang}
\ead{yuguang.wang@unsw.edu.au}

\author[addrcityu]{Xiaosheng Zhuang}
\ead{xzhuang7@cityu.edu.hk}

%
%%\cortext[corau]{Corresponding author.}
%
%\address[addrltu]{Department of Mathematics and Statistics,
%La Trobe University, Melbourne, VIC, Australia}
\address[addrzjnu]{Department of Educational Technology, Zhejiang Normal University, Jinhua, China}
\address[addrprin]{Department of Physics, Princeton University, New Jersey, USA}
\address[addrunsw]{School of Mathematics and Statistics,
The University of New South Wales, Sydney, Australia}
\address[addrcityu]{Department of Mathematics, City University of Hong Kong, Hong Kong}

\begin{abstract}
	Graph Neural Networks (GNNs) have become a topic of intense research recently due to their powerful capability in high-dimensional classification and regression tasks for graph-structured data. However, as GNNs typically define the graph convolution by the orthonormal basis for the graph Laplacian, they suffer from high computational cost when the graph size is large. This paper introduces Haar basis which is a sparse and localized orthonormal system for a coarse-grained chain on graph. The graph convolution under Haar basis, called Haar convolution, can be defined accordingly for GNNs. The sparsity and locality of the Haar basis allow Fast Haar Transforms (FHTs) on graph, by which a fast evaluation of Haar convolution between graph data and filters can be achieved. We conduct experiments on GNNs equipped with Haar convolution, which demonstrates state-of-the-art results on graph-based regression and node classification tasks.
\end{abstract}


\end{frontmatter}

%%%%%%%%%%%%%%%%%%%%%%%%%%%%%%%%%%%%%%%%%%%%%%%%%%%%%%%%%%%%%%%%%%%
%% section of introduction
\section{Introduction}
\label{intro}
Convolutional neural networks (CNNs) have been a very successful machinery in many high-dimensional regression and classification tasks on Euclidean domains \cite{KrSuHi2012,LeBeHi2015}. Recently, its generalization to non-Euclidean domains, known as \emph{geometric deep learning}, has attracted growing attention, due to its great potential in pattern recognition and regression for graph-structured data, see \cite{Bronstein_etal2017}.

Graph neural networks (GNNs) are a typical model in geometric deep learning, which replaces the partial derivatives in CNNs by the Laplacian operator \cite{BrZaSzLe2013,HeBrLe2015}.
The Laplacian, which carries the structural features of the data, is a second-order isotropic differential operator that admits a natural generalization to graphs and manifolds. In GNNs, input data are convoluted with filters  under an orthonormal system for the Laplacian. However, as the algebraic properties of regular Euclidean grids are lost in general manifolds and graphs, FFTs (fast Fourier transforms) for the Laplacian are not available. This leads to the issue that the computation of convolution for graph data is not always efficient, especially when the graph dataset is large.

In this paper, we introduce an alternative orthonormal system on
graph, the \emph{Haar basis}. It then defines a new graph convolution
for GNNs --- \emph{Haar convolution}. Due to the sparsity and locality of the Haar basis, fast Haar transforms (FHTs) can be achieved on graph-structured data. This  significantly improves the computational efficiency of GNNs as the Haar convolution guarantees the linear computational complexity. We apply Haar convolution to GNNs and give a novel type of deep convolutional neural networks on graph --- HANet. Numerical tests on real graph datasets show that HANet achieves good performance and computational efficiency in classification and regression tasks. To the best of our knowledge, our method is the first fast algorithm for spectral graph convolution by appropriately selecting orthogonal basis on graph, which is of great importance in the line of building spectral-based GNN models. Overall, the major contributions of the paper are summarized as three-fold.

\begin{itemize}
\item The Haar basis is introduced for graphs. Both theoretical analysis and real examples of the sparsity and locality are given. With these properties, the fast algorithms for Haar transforms (FHTs) are developed and their complexity analysis is studied.
\item The Haar convolution under Haar basis is developed. By virtue of FHTs, the computational cost for Haar convolution is proportional to the size of graph, which is more efficient than Laplacian-based spectral graph convolution. Other technical components, including weight sharing and detaching, chain and pooling, are also presented in details.
\item GNN with Haar convolution (named HANet) is proposed. The experiments illustrate that HANet with high efficiency achieves good performance on a broad range of high-dimensional regression and classification problems on graphs.
\end{itemize}

The paper is organized as follows. In Section~\ref{sec:relatedwork}, we review recent advances on GNNs. In Section~\ref{sec:gnn_with_haar}, we construct the Haar orthonormal basis using a chain on the graph. The Haar basis will be used to define a new graph convolution, called Haar convolution. In Section~\ref{sec:fastalgo}, we develop fast algorithms for Haar transforms and the fast Haar transforms allows fast computation of Haar convolution. In Section~\ref{sec:hanet}, we use the Haar convolution as the graph convolution in graph neural networks. Section~\ref{sec:experiments} shows the experimental results of GNNs with Haar convolution (HANet) on tasks of graph-based regression and node classification.

\section{Related Work}\label{sec:relatedwork}
Developing deep neural networks for graph-structured data has received extensive attention in recent years \cite{scarselli2009graph,LiTaBrZe2015,duvenaud2015convolutional,niepert2016learning,Survey_Battaglia,Survey_SunMS,Survey_ZhangCQ,
Survey_ZhuWW,nickel2015review,goyal2018graph,hamilton2017representation,chen2017graph,han2017laplacian,da2017tree,Scarselli_etal2009,wang2019haarpooling}.
%Developing deep neural networks for graph-structured data has received extensive attention in recent years \cite{scarselli2009graph,LiTaBrZe2015,duvenaud2015convolutional,niepert2016learning,Survey_Battaglia,Survey_SunMS,Survey_ZhangCQ,Survey_ZhuWW,nickel2015review,goyal2018graph,hamilton2017representation}. 
Bruna et al. \cite{BrZaSzLe2013} first propose graph convolution, which is defined by graph Fourier transforms under the orthogonal basis from the graph Laplacian. The graph convolution uses Laplacian eigendecomposition which is computationally expensive. Defferrard et al. \cite{DeBrVa2016} approximate smooth filters in the spectral domain by Chebyshev polynomials. Kipf and Welling \cite{KiWe2017} simplify the convolutional layer by exploiting first-order Chebyshev polynomial for filters. Following this line, several acceleration methods for graph convolutional networks are proposed  \cite{ChZhSo2018,ChMaXi2018fastgcn}. Graph wavelet neural networks \cite{GWNN} replace graph Fourier transform by graph wavelet transform in the graph convolution, where Chebyshev polynomials are used to approximate the graph wavelet basis \cite{hammond2011wavelets}. Although GWNN circumvents the Laplacian eigendecomposition, the matrix inner-product operations are nevertheless not avoidable in wavelet transforms for convolution computation.

Graph convolutional networks with attention mechanisms \cite{Attention_GCN,velivckovic2017graph} can effectively learn the importance between nodes and their neighbors, which is more suitable for node classification task (than graph-based regression). But much computational and memory cost is required to perform the attention mechanism in the convolutional layers. Yang et al. \cite{Yang_etal2019} propose Shortest Path Graph Attention Network (SPAGAN) by using path-based attention mechanism in node-level aggregation, which leads to superior results than GAT \cite{velivckovic2017graph} concerning neighbor-based attention.

Some GNN models \cite{LNet,SGC,N-GCN2019} use multi-scale information and higher order adjacency matrix to define graph convolution. To increase the scalability of the model for large-scale graph, Hamilton et al. \cite{HaYiLe2017} propose the framework Graph-SAGE with sampling and a neural network based aggregator over a fixed size node neighbor. Artwood and Twosley develope diffusion convolutional neural networks \cite{DCNN_2016} by using diffusion operator for graph convolution. MoNet \cite{Monti_etal2017} introduces a general methodology to define spatial-based graph convolution by the weighted average of multiple weighting functions on neighborhood.
Gilmer et al. \cite{Gilmer_etal2017} provide a unified framework, the Message Passing Neural Networks (MPNNs), by which some existing GNN models are incorporated. Xu et al. \cite{GIN} present a theoretical analysis for the expressive power of GNNs and propose a simple but powerful variation of GNN, the graph isomorphism network. By generalizing the graph Laplacian to maximal entropy transition matrix derived from a path integral, \cite{MaLiWa2019} proposes a new framework called PAN that involves every path linking the message sender and receiver with learnable weights depending on the path length.

%\textbf{Haar Basis and Fast Algorithms.} Haar basis rooted in the theory of Haar wavelet basis as first introduced by Haar \cite{Haar1910}, is a special case of Daubechies wavelets \cite{Daubechies1992}, and later developed onto graph by Belkin et al. \cite{BeNiSi2006}, see also \cite{ChFiMh2015}. The fast computation for sparse matrix multiplication is well-studied (see e.g. \cite{GoVa2012}). The sparse Fourier transforms by which sparse Haar transforms can be developed are seen in \cite{InKa2014,InKaPr2014,HaInKaPr2012simple}.

\section{Graph convolution with Haar basis}\label{sec:gnn_with_haar}
\subsection{Graph Fourier Transform}
Bruna et al. \cite{BrZaSzLe2013} first defined the graph convolution based on spectral graph theory \cite{ChGr1997} and the graph Laplacian. An un-directed weighted graph $\gph=(V,E,w)$ is a triplet  with vertices $V$, edges $E$ and weights $w:E\to \R$. Denote by $N:=|V|$ the number of vertices of the graph.
Let $l_2(\gph):=\{f:V\to \R \,|\, \sum_{v\in V}|f(v)|^2<\infty\}$ the real-valued $l_2$ space on the graph with inner product $f\cdot g:=\sum_{v\in V}f(v)g(v)$. A basis for $l_2(\gph)$ is a set of vectors $\{u_{\ell}\}_{\ell=1}^N$ on $\gph$ which are linearly independent and orthogonal (i.e. $u_{\ell}\cdot u_{\ell'}=0$ if $\ell\neq\ell'$). The (normalized) eigenvectors $\{u_{\ell}\}_{\ell=1}^{|V|}$ of the graph Laplacian $\mathcal{L}$ forms an orthonormal basis for $l_2(\gph)$. We call the matrix $U:=(u_{1},\dots,u_{N})$ the \emph{(graph Fourier)} base matrix, whose columns form the \emph{graph Fourier basis} for $l_2(\mathcal G)$. The \emph{graph convolution} can then be defined by
\begin{equation}\label{eq:graphconvL}
	g\star f = U \bigl((U^T g)\odot (U^T f)\bigr),
\end{equation}
where $U^T f$ is regarded as the \emph{adjoint discrete graph Fourier transform} of $f$, $U c$ is the \emph{forward discrete graph Fourier transform} of $c$ on $\gph$ and $\odot$ is the element-wise Hadamard product.

While graph convolution defined in \eqref{eq:graphconvL} is conceptually important, it has some limitations in practice. First, the base matrix $U$ is obtained by using eigendecomposition of the graph Laplacian in the sense that $\mathcal{L}=U \Lambda U^{T}$, where $\Lambda$ is the diagonal matrix of corresponding eigenvalues. The computational complexity is proportional to $\mathcal{O}(N^3)$, which is impractical when the number of vertices of the graph is quite large. Second, the computation of the forward and inverse graph Fourier transforms (i.e. $U^{T}f$ and $U c$) have $\mathcal{O}(N^2)$ computational cost due to the multiplication by (dense) matrices $U$ and $U^T$. In general, there is no fast algorithms for the graph Fourier transforms as the graph nodes are not regular and the matrix $U$ is not sparse. Third, filters in the spectral domain cannot guarantee the localization in the spatial (vertex) domain, and $O(Ndm)$ parameters need to be tuned in the convolutional layer with $m$ filters (hidden nodes) and $d$ features for each vertex.

To alleviate the cost of computing the graph Fourier transform, Chebyshev polynomials \cite{DeBrVa2016} are used to construct localized polynomial filters for graph convolution, where the resulting graph neural network is called ChebNet. Kipf and Welling \cite{KiWe2017} simplify ChebNet to obtain graph convolutional networks (GCNs). However, such a polynomial-based approximation strategy may lose information in the spectral graph convolutional layer, and matrix multiplication is still not avoidable as FFTs are not available for graph convolution. Thus, the graph convolution in this scenario is also computationally expensive, especially for dense graph of large size.
We propose an alternative orthonormal basis that allows fast computation for the corresponding graph convolution, which then improves the scalability and efficiency of existing graph models. The basis we use is the Haar basis on a graph. The Haar basis replaces the matrix of eigenvectors $U$ in \eqref{eq:graphconvL} and forms a highly sparse matrix, which reflects the clustering information of the graph. The sparsity of the Haar transform matrix allows fast computation (in nearly linear computational complexity) of the corresponding graph convolution.

\subsection{Haar Basis}\label{subsec:haarbasis}
Haar basis rooted in the theory of Haar wavelet basis as first introduced by Haar \cite{Haar1910}, is a special case of Daubechies wavelets \cite{Daubechies1992}, and later developed onto graph by Belkin et al. \cite{BeNiSi2006}, see also \cite{ChFiMh2015}. 
The construction of the Haar basis exploits a chain of the graph.
For a graph $\gph=(V,E,w)$, a graph $\gph^{\rm cg}:=(V^{\rm cg},E^{\rm cg},w^{\rm cg})$ is called a \emph{coarse-grained graph} of $\gph$ if $|V^{\rm cg}|\leq |V|$ and each vertex of $\gph$ associates with exactly one (parent) vertex in $\gph^{\rm cg}$. Each vertex of $\gph^{\rm cg}$ is called a \emph{cluster} of $\gph$.
Let  $J_0,J$ be two integers such that $J>J_0$. A \emph{coarse-grained chain} for $\gph$ is a set of graphs $\gph_{J\to J_0}:=(\gph_{J},\gph_{J-1},\dots,\gph_{J_0})$ such that $\gph_J=\gph$ and $\gph_{j}$ is a coarse-grained graph of $\gph_{j+1}$ for $j=J_0,J_0+1,\dots,J-1$.  $\gph_{J_0}$ is the \emph{top level} or the \emph{coarsest level} graph while $\gph_{J}$ is the \emph{bottom level} or the \emph{finest level} graph. If the top level $\gph_{J_0}$ of the chain has only one node, $\gph_{J\to J_0}$ becomes a tree.
The chain $\gph_{J\to J_0}$ gives a hierarchical partition for the graph $\gph$. For details about graphs and chains, see examples in \cite{ChGr1997,HaVaGr2011,ChFiMh2015,ChMhZh2018,WaZh2018,WaZh2019}.

\textbf{Construction of Haar basis.} With a chain of the graph, one can generate a Haar basis for $l_2(\gph)$ following \cite{ChFiMh2015}, see also \cite{GaNaCo2010}. We show the construction of Haar basis on $\gph$, as follows.

\textbf{Step~1.} Let $\gph^{\rm cg}=(V^{\rm cg},E^{\rm cg},w^{\rm cg})$ be a coarse-grained graph of $\gph=(V,E,w)$ with $N^{\rm cg}:=|V^{\rm cg}|$. Each vertex $v^{\rm cg}\in V^{\rm cg}$ is a cluster $v^{\rm cg}=\{v\in V\,|\,  v\mbox{ has parent } v^{\rm cg}\}$ of $\gph$. Order $V^{\rm cg}$, e.g., by degrees of vertices or weights of vertices, as $V^{\rm cg}=\{v^{\rm cg}_1,\ldots,v^{\rm cg}_{N^{\rm cg}}\}$. We define $N^{\rm cg}$ vectors $\eigfm^{\rm cg}$ on $\gph^{\rm cg}$ by
\begin{equation}\label{eq:haargc1}
\eigfm[1]^{\rm cg}(v^{\rm cg})  :=\frac{1}{\sqrt{N^{\rm cg}}},
\quad v^{\rm cg}\in V^{\rm cg},
\end{equation}
and for $\ell=2,\ldots,N^{\rm cg}$,  
\begin{equation}\label{eq:haargc2}
\eigfm[\ell]^{\rm cg}:=\sqrt{\frac{N^{\rm cg}-\ell+1}{N^{\rm cg}-\ell+2}}\left(\chi^{\rm cg}_{\ell-1}-\frac{\sum_{j=\ell}^{N^{\rm cg}}\chi_j^{\rm cg}}{{N^{\rm cg}-\ell+1}}\right),
\end{equation}
where  $\chi_j^{\rm cg}$ is the indicator function for the $j$th vertex $v_j^{\rm cg}\in V^{\rm cg}$ on $\gph$ given by
\[
\chi_j^{\rm cg}(v^{\rm cg}) :=
\begin{cases}
1, & v^{\rm cg} = v_j^{\rm cg},\\
0, & v^{\rm cg}\in V^{\rm cg}\backslash \{v_j^{\rm cg}\}.
\end{cases}
\]
Then, the set of functions $\{\phi_{\ell}^{\rm cg}\}_{\ell=1}^{N^{\rm cg}}$ forms an orthonormal basis for $l_2(\gph^{\rm cg})$.

Note that  each $v\in V$ belongs to exactly one cluster $v^{\rm cg}\in V^{\rm cg}$. In view of this, 
for each $\ell=1,\dots,N^{\rm cg}$, we extend the vector $\phi_\ell^{\rm cg}$ on $\gph^{\rm cg}$ to a vector $\eigfm[\ell,1]$ on $\gph$ by
\begin{equation*}%\label{defn:Haar-orth-gph-1}
\eigfm[\ell,1](v):=
\frac{\eigfm^{\rm cg}(v^{\rm cg})}{\sqrt{|v^{\rm cg}|}}, \quad  v\in  v^{\rm cg},
\end{equation*}
here $|v^{\rm cg}|:=k_\ell$ is the size of the cluster $v^{\rm cg}$, i.e., the number of vertices in $\gph$ whose common parent is $v^{\rm cg}$.  We  order the cluster $v_\ell^{\rm cg}$,  e.g., by degrees of vertices, as
\[
v_\ell^{\rm cg} = \{v_{\ell,1},\ldots,v_{\ell,k_\ell}\}\subseteq V.
\]
For $k=2,\ldots,k_\ell$, similar to \eqref{eq:haargc2}, define
\begin{equation*}%\label{defn:Haar-orth-gph-2}
\eigfm[\ell,k] =
 \sqrt{\frac{k_\ell-k+1}{k_\ell-k+2}}\left(\chi_{\ell,k-1}-\frac{\sum_{j=k}^{k_\ell}\chi_{\ell,j}}{k_\ell-k+1}\right).
\end{equation*}
where for $j=1,\dots,k_{\ell}$, $\chi_{\ell,j}$  is given by
\[
\chi_{\ell,j}(v) :=
\begin{cases}
1, & v = v_{\ell,j},\\
0, & v \in V\backslash \{v_{\ell,j}\}.
\end{cases}
\]
One can verify that the resulting $\{\phi_{\ell,k}: \ell=1,\dots,N^{\rm cg}, k=1,\dots,k_{\ell}\}$ is an orthonormal basis for $l_2(\gph)$.

\textbf{Step~2.} Let $\gph_{J\to J_0}$ be a coarse-grained chain for the graph $\gph$. An orthonormal basis $\{\phi_{\ell}^{(0)}\}_{\ell=1}^{N_0}$ for $l_2(\gph_{J_0})$ is generated using \eqref{eq:haargc1} and \eqref{eq:haargc2}. We then
repeatedly use Step~1:
for $j=J_0+1,\dots,J$, we generate an orthonormal basis $\{\phi_{\ell}^{(j)}\}_{\ell=1}^{N_j}$ for $l_2(\gph_j)$ from the orthonormal basis $\{\phi_{\ell}^{(j-1)}\}_{\ell=1}^{N_{j-1}}$ for the coarse-grained graph $\gph_{j-1}$ that was derived in the previous steps. We call the sequence $\{\phi_{\ell}:=\phi_{\ell}^{(J)}\}_{\ell=1}^N$ of  vectors at the finest level,  the \emph{Haar global orthonormal basis} or simply the \emph{Haar basis} for $\gph$ associated with the chain $\gph_{J\to J_0}$. The orthonormal basis $\{\phi_{\ell}^{(j)}\}_{\ell=1}^{N_j}$ for $l_2(\gph_j)$, $j=J-1,J-2,\dots,J_0$ is called the \emph{associated (orthonormal) basis} for the Haar basis $\{\phi_{\ell}\}_{\ell=1}^{N}$.

\begin{proposition}\label{prop:orthogonality}
For each level $j=J_0,\dots,J$, the sequence $\{\phi_{\ell}^{(j)}\}_{\ell=1}^{N_j}$ is an orthonormal basis for $l_2(\gph_j)$, and in particular, $\{\phi_{\ell}\}_{\ell=1}^N$ is an orthonormal basis for $l_2(\gph)$; each basis $\{\phi_{\ell}^{(j)}\}_{\ell=1}^{N_j}$ is the Haar basis for the chain $\gph_{j\to J_0}$.
\end{proposition}

\begin{proposition}\label{prop:spochaar}
	Let $\gph_{J \to J_0}$ be a coarse-grained chain for $\gph$. If each parent of level $\gph_j$, $j=J-1,J-2,\dots,J_0$, contains at least two children, the number of different values of the Haar basis $\phi_{\ell}$, $\ell=1,\dots,N$, is bounded by a constant.
\end{proposition}

\begin{figure}[t]
%\vskip 0.2in
\begin{center}
\begin{minipage}{\columnwidth}
\centering
\begin{minipage}{0.4\columnwidth}
	\includegraphics[width=0.88\columnwidth]{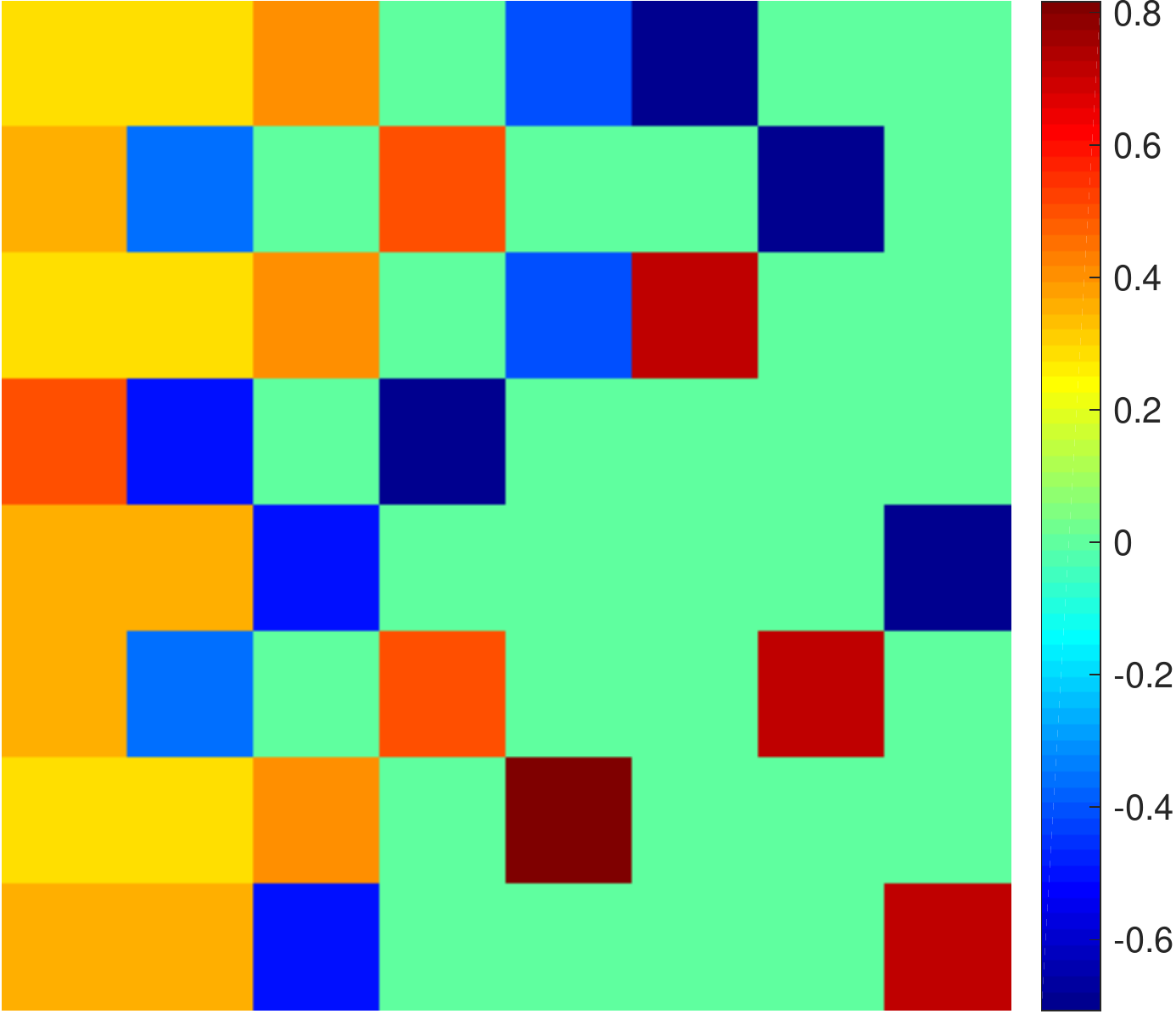}\\
	\centering
	{\scriptsize (a) Haar transform matrix}
\end{minipage}
\hskip -0.1in
\begin{minipage}{0.56\columnwidth}
	\includegraphics[width=1\columnwidth]{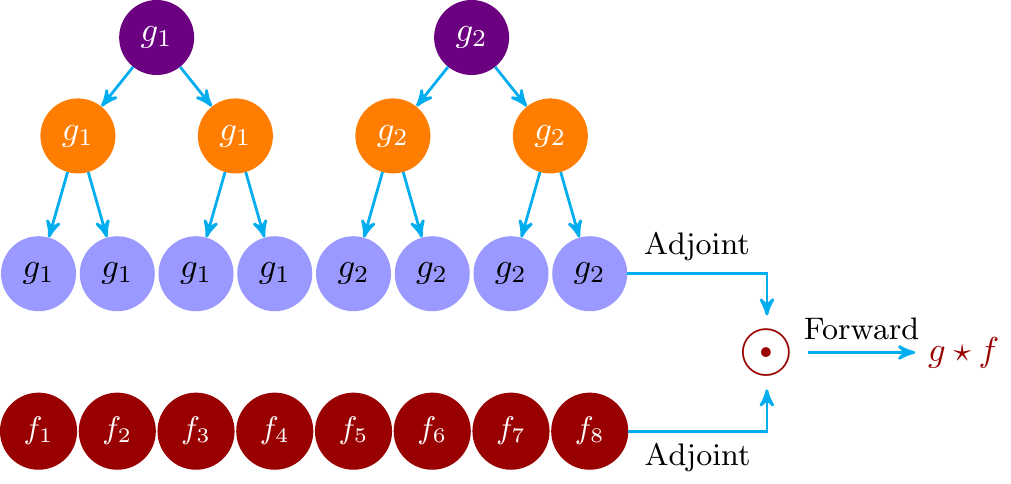}\\[5mm]
\centering	
	{\scriptsize (b) Haar convolution}
\end{minipage}
\end{minipage}
\vskip -0.015in
\caption{(a) The $8\times 8$ matrix $\Phi$ of the \emph{Haar Basis} for a graph $\gph$ with $8$ nodes. The green entries are zero and the matrix $\Phi$ is sparse. The Haar basis is created based on the coarse-grained chain $\gph_{2\to 0}:=(\gph_2,\gph_1,\gph_0)$, where $\gph_2,\gph_1,\gph_0$ are graphs with $8,4,2$ nodes. For $j=1,2$, each node of $\gph_{j-1}$ is a cluster of nodes in $\gph_{j}$. Each column of $\Phi$ is a member of the Haar basis. The first two columns can be compressed as an orthonormal basis of $\gph_0$ and the first to fourth columns can be reduced to the orthonormal basis for $\gph_1$. (b) \emph{Haar Convolution} $g\star f$ using the Haar basis of (a), where the weight sharing for filter vector $g$ is defined by the chain $\gph_{2\to0}$ and the $g\star f$ is the forward Haar transform of the point-wise product of the adjoint Haar transforms of $g$ and $f$, where the Haar transforms have a fast algorithmic implementation.}
\label{fig:haarbasis1}
\end{center}
\vskip -0.23in
\end{figure}

The Haar basis depends on the chain for the graph. If the topology of the graph is well reflected by the clustering of the chain, the Haar basis then contains the crucial geometric information of the graph. For example, by using $k$-means clustering algorithm \cite{Lloyd1982} or METIS algorithm \cite{KaKu1998} one can generate a chain that reveals desired geometric properties of the graph.

%Figure~\ref{fig:chain1}
Figure~\ref{fig:haarbasis1}b shows a chain $\gph_{2\to 0}$ with $3$ levels of a graph $\gph$. Here, for each level, the vertices are given by
\begin{align*}
	V^{(2)}&=V= \{v_1,\dots,v_8\},\\
	V^{(1)}&= \{v_1^{(1)},v_2^{(1)},v_3^{(1)},v_4^{(1)}\}\\
	&=\{\{v_1,v_2\},\{v_3,v_4\},\{v_5,v_6\},\{v_7,v_8\}\},\\
	V^{(0)}&= \{v_1^{(0)},v_2^{(0)}\}
	=\{\{v_1^{(1)},v_2^{(1)}\},\{v_3^{(1)},v_4^{(1)}\}\}.
\end{align*}
Figure~\ref{fig:haarbasis1}a shows the Haar basis for the chain $\gph_{2\to 0}$. There are in total $8$ vectors of the Haar basis for $\gph$. From construction, the Haar basis $\eigfm$ and the associated basis $\eigfm^{(j)}$, $j=1,2$ are closely connected: the $\eigfm[1],\eigfm[2]$ can be reduced to $\eigfm[1]^{(0)},\eigfm[2]^{(0)}$ and the $\eigfm[1],\eigfm[2],\eigfm[3],\eigfm[4]$ can be reduced to $\eigfm[1]^{(1)},\eigfm[2]^{(1)},\eigfm[3]^{(1)},\eigfm[4]^{(1)}$. This connection would allow fast algorithms for Haar transforms as given in Algorithms~\ref{alg:adft} and \ref{alg:dft}.
In Figure~\ref{fig:haarbasis1}, the matrix $\Phi^T$ of the $8$ Haar basis vectors $\eigfm$ on $\gph$ has good sparsity. With the increase of the graph size, the sparsity of the Haar transform matrix $\Phi$ becomes more prominent, which we will demonstrate in the experiments in Section~\ref{subsec:experimentFHT}.

\subsection{Haar Convolution}\label{sec:haarconv}
With the Haar basis constructed in Section~\ref{subsec:haarbasis}, we can define Haar convolution as an alternative form of spectral graph convolution in \eqref{eq:graphconvL}.
Let $\{\phi_{\ell}\}_{\ell=1}^{N}$ be the Haar basis associated with a chain $\gph_{J\to J_0}$ of a graph $\gph$.
Denoted by $\Phi=(\phi_1,\dots,\phi_N)\in \R^{N\times N}$ the Haar transform matrix. We define by
\begin{equation}\label{eq:adft}
	\Phi^T f = \left(\sum_{v\in V}\phi_1(v)f(v),\dots,\sum_{v\in V}\phi_N(v)f(v)\right)\in\R^N
\end{equation}
the \emph{adjoint Haar transform} for graph data $f$ on $\gph$, and by
\begin{equation}\label{eq:dft}
	(\Phi c)(v) = \sum_{\ell=1}^N\phi_{\ell}(v) c_{\ell},\quad v\in V,
\end{equation}
the \emph{forward Haar transform} for (coefficients) vector $c:=(c_1,\dots,c_N)\in \R^N$. We call the matrix $\Phi$ \emph{Haar transform matrix}.

\begin{definition}
	The Haar convolution for filter $g$ and graph data $f$ on $\gph$ can be defined as
\begin{equation}\label{defn:graphconv}
	g\star f = \Phi((\Phi^T g)\odot (\Phi^T f)).
\end{equation}
\end{definition}
Computationally, \eqref{defn:graphconv} is obtained by performing forward Haar transform of the element-wise Hadamard product between adjoint Haar transform of $g$ and $f$.
Compared with the Laplacian based spectral graph convolution given in \eqref{eq:graphconvL}, the Haar convolution has the following  features.
(i) the Haar transform matrix $\Phi$ is sparse and the computation of $\Phi^T f$ or $\Phi c$ is more efficient than
$U^T f$ or $Uc$;
(ii) as the Haar basis is constructed based on the chain of the graph which reflects the clustering property for vertices, the Haar convolution can extract abstract features for input graph data, that is, it provides a learning representation for graph-structured data;
(iii) by means of the sparsity of Haar basis, the adjoint and forward Haar transforms can be implemented by fast algorithms, which have nearly linear computational complexity (with respect to the size of the input graph).

We can presume the filter in the ``frequency domain'' and skip adjoint Haar transform of filter $g$ (i.e. $\Phi^T g$), and then write Haar convolution as $g\star f = \Phi(g\odot (\Phi^T f))$.

\subsection{Fast Algorithms for Haar Transforms and Haar Convolution} The computation of Haar transforms can also be accelerated by using sparse matrix multiplications due to the sparsity of the Haar transform matrix. This would allow the linear computational complexity $O(\epsilon N)$ with sparsity $1-\epsilon$ of the Haar transform matrix. Moreover, a similar computational strategy to the sparse Fourier transforms \cite{HaInKaPr2012simple,InKaPr2014} can be applied so that the Haar transforms achieve faster implementation with time complexity $O(k\log N)$ for graph with $N$ nodes and the Haar transform matrix with $k$ non-zero elements.
By the sparsity of Haar transform matrix, fast Haar transforms (FHTs) which includes \emph{adjoint Haar transform} and \emph{forward Haar transform} can be developed to speed up the implementation of Haar convolution. Theorems~\ref{thm:adft} and \ref{thm:fdft} in the following section show that the computational cost of adjoint and forward Haar transform can reach $\mathcal{O}(N)$ and $\mathcal{O}(N(\log N)^2)$. They are nearly linear computational complexity and are thus called \emph{fast Haar transforms} (FHTs).
The Haar convolution in \eqref{defn:graphconv} consists of two adjoint Haar transforms and a forward Haar transform, and can then be evaluated in $\mathcal{O}(N(\log N)^2)$ steps.

\subsection{Weight Sharing}
We can use weight sharing in Haar convolution to reduce the number of parameters of the filter, and capture the common feature of the nodes which lie in the same cluster. As the resulting clusters contain information of neighbourhood, we can use the chain $\gph_{J\to J_0}$ for weight sharing: the vertices of the graph which have the same parent at a coarser level share a parameter of the filter. Here, the coarser level is some fixed level $J_1$, $J_0\leq J_1<J$. For example, the weight sharing rule for chain $\gph_{2\to0}$ in Figure~\ref{fig:haarbasis1}b is: assign the weight $g_i$ for each node $v_i^{(0)}$, $i=1,2$ on the top level, the filter (or the weight vector) at the bottom level is then $g=(g_1,g_1,g_1,g_1,g_2,g_2,g_2,g_2)$.
In this way, one has used the filter $g$ with two independent parameters $g_1, g_2$ to convolute with the input vector with $8$ components.

\section{Fast algorithms under Haar basis}\label{sec:fastalgo}
For the Haar convolution introduced in Definition 3 (see Eq. \ref{defn:graphconv}), we can develop an efficient computational strategy by virtue of the sparsity of the Haar transform matrix.
Let $\gph_{J\to J_0}$ be a coarse-grained chain of the graph $\gph$. For convenience, we label the vertices of the level-$j$ graph $\gph_j$ by $V_j:=\bigl\{\vj_1,\ldots,\vj_{N_j}\bigr\}$.

\subsection{Fast Computation for Adjoint Haar Transform $\Phi^{T}f$}
 The adjoint Haar transform in \eqref{eq:adft} can be computed in the following way.
For $j=J_0,\dots,J-1$, let $c_{k}^{(j)}$ be the number of children of $\vj_k$, i.e. the number of vertices of $\gph_{j+1}$ which belongs to the cluster $\vj_k$, for $k=1,\ldots,N_j$. For $j=J$, let $c_{k}^{(J)}\equiv1$ for $k=1,\dots,N$. For $j=J_0,\dots,J$ and $k=1,\dots,N_j$, we define the weight factor for $\vj_k$ by
\begin{equation}\label{eq:wjk}
	\wj_k:=\frac{1}{\sqrt{c_{k}^{(j)}}}.
\end{equation}
Let $W_{J\to J_0}:=\{\wj_k\, | \, j=J_0,\dots,J,\: k=1,\dots,N_j\}$.
Then, the weighted chain $(\gph_{J\to J_0},W_{J\to J_0})$ is a \emph{filtration} if each parent in the chain $\gph_{J\to J_0}$ has at least two children. See e.g. \cite[Definition~2.3]{ChFiMh2015}.

Let $\{\phi_{\ell}\}_{\ell=1}^{N}$ be the Haar basis obtained in Step 2 of Section~\ref{subsec:haarbasis}, which we also call the Haar basis for the filtration $(\gph_{J\to J_0},W_{J\to J_0})$ of a graph $\gph$. We define the weighted sum for $f\in l_2(\gph)$ by
\begin{equation}\label{eq:ws1}
	\ws[J]\bigl(f,\vj[J]_k\bigr):= f(\vj[J]_k), \quad \vj[J]_k\in \gph_J,
\end{equation}
and for $j=J_0,\dots,J-1$ and $\vj[j]_k\in \gph_j$,
\begin{equation}\label{eq:wsj}
	\ws\bigl(f,\vj_k\bigr):= \sum_{\vj[j+1]_{k'}\in \vj_k}\wj[j+1]_{k'} \ws[j+1]\bigl(f,\vj[j+1]_{k'}\bigr).
\end{equation}
For each vertex $\vj_k$ of $\gph_j$, the $\ws\bigl(f,\vj_k\bigr)$ is the weighted sum of the $\ws[j+1]\bigl(f,\vj[j+1]_{k'}\bigr)$ at the level $j+1$ for those vertices $\vj[j+1]_{k'}$ of $\gph_{j+1}$ whose parent is $\vj_k$.

The adjoint Haar transform can be evaluated by the following theorem.
\begin{theorem}\label{thm:adft} Let $\{\eigfm\}_{\ell=1}^{N}$ be the Haar basis for the filtration $(\gph_{J\to J_0},W_{J\to J_0})$ of a graph $\gph$. Then, the adjoint Haar transform for the vector $f$ on the graph $\gph$ can be computed by, for $\ell=1,\dots,N$,
\begin{equation}\label{eq:adftbyws}
	(\Phi^T f)_{\ell} = \sum_{k=1}^{N_j} \ws\bigl(f,\vj_k\bigr)\wj_k \eigfm^{(j)}(\vj_k),
\end{equation}
where $j$ is the smallest possible number in $\{J_0,\ldots,J\}$ such that  $\eigfm^{(j)}$ is the $\ell$th member of the orthonormal basis $\{\eigfm^{(j)}\}_{\ell=1}^{N_j}$ for $l_2(\gph_j)$ associated with the Haar basis $\{\eigfm\}_{\ell=1}^N$ (see Section~\ref{subsec:haarbasis}), $\vj_k$ are the vertices of $\gph_j$ and weights $\wj_k$ are given by \eqref{eq:wjk}.
\end{theorem}
%The proof of Theorem~\ref{thm:adft} is given in Supplementary Material.
\begin{proof}[Proof]
	By the relation between $\eigfm$ and $\eigfm^{(j)}$,
	\begin{align*}
		(\Phi^T f)_{\ell} &= \sum_{k=1}^{N}f(\vj[J]_k)\eigfm(\vj[J]_k)\\
		&= \sum_{k'=1}^{N_{J-1}} \left(\sum_{\vj[J]_k\in \vj[J-1]_{k'}}f(\vj[J]_k)\right)\wj[J-1]_{k'}\eigfm^{(J-1)}(\vj[J-1]_{k'})\\
		&=\sum_{k'=1}^{N_{J-1}} \ws[J-1](f,\vj[J-1]_{k'})\wj[J-1]_{k'}\eigfm^{(J-1)}(\vj[J-1]_{k'})\\
		&=\sum_{k''=1}^{N_{J-2}} \left(\sum_{\vj[J-1]_{k'}\in\vj[J-2]_{k''}}\ws[J-1](f,\vj[J-1]_{k'})\wj[J-1]_{k'}\right)\\
		&\quad \times\wj[J-2]_{k''}\eigfm^{(J-2)}(\vj[J-2]_{k''})\\
		&=\sum_{k''=1}^{N_{J-2}} \ws[J-2](f,\vj[J-2]_{k''})\wj[J-2]_{k''}\eigfm^{(J-2)}(\vj[J-2]_{k''})\\
		&\cdots\\
		&=\sum_{k=1}^{N_j} \ws(f,\vj_k)\wj_k\eigfm^{(j)}(\vj_k),
	\end{align*}
	where  we recursively compute the summation to obtain the last equality,
	thus completing the proof.
\end{proof}

\subsection{Fast Computation for Forward Haar Transform $\Phi c$} The forward Haar transform in \eqref{eq:dft} can be computed, as follows.
\begin{theorem}\label{thm:fdft}
Let $\{\eigfm\}_{\ell=1}^{N}$ be the Haar basis for a filtration ($\gph_{J\to J_0}$,$W_{J\to J_0}$) of graph $\gph$ and $\{\eigfm^{(j)}\}_{\ell=1}^{N_j}$, $j=J_0,\dots,J$ be the associated bases at $\gph_j$. Then, the forward Haar transform for vector $c=(c_1,\dots,c_N)\in\R^N$ can be computed by, for $k=1,\dots,N$,
\begin{equation*}
(\Phi c)_k =\sum_{j=1}^{J} W_{k}^{(j)}\left(\sum_{\ell=N_{j-1}+1}^{N_j}  c_\ell \eigfm^{(j)}(\vj_{k_j})\right),
\end{equation*}
where for $k=1,\dots,N$, $\vj_{k_j}$ is the parent (ancestor) of $\vj[J]_k$ at level $j$, and $W_{k}^{(J)} := 1$ and
\begin{equation}\label{eq:Wkj}
	W_{k}^{(j)} := \prod_{n=2}^{j}\wj[n]_{k_n} \;\mbox{~for~} j=J_0,\dots,J-1,
\end{equation}
where the weight factors $\wj[n]_{k_n}$ for $n=1,\dots,J$ are given by \eqref{eq:wjk}.
\end{theorem}
\begin{proof}[Proof]
Let $N_j:=|V_j|$ for $j=J_0,\ldots,J$ and $N_{J_0-1}:=0$.
For $k=1,\dots,N_J$, let $\vj[J]_k$ the $k$th vertex of $\gph_J$. For $i=J_0,\dots,J-1$, there exists $k_i=1,\dots,N_j$ such that $\vj[i]_{k_i}$ the parent at level $i$ of $\vj[J]_k$. By the property of the Haar basis, for each vector $\eigfm$ there exists $j\in\{J_0,\dots,J\}$ such that $\ell\in\{N_{j-1}+1,\dots,N_j\}$, $\eigfm$ is a constant for the vertices of $\gph_J=\gph$ which have the same parent at level $j$. Then,
\begin{align}\label{eq:v1byvj}
 \eigfm(\vj[J]_k)
 &= \wj[J-1]_{k_{J-1}} \eigfm^{(J-1)}(\vj[J-1]_{k_{J-1}})\notag\\
 &= \wj[J-1]_{k_{J-1}}\wj[J-2]_{k_{J-2}} \eigfm^{(J-2)}(\vj[J-2]_{k_{J-2}})\notag\\
 &= \left(\prod_{n=J_0}^{j}\wj[n]_{k_n}\right)\eigfm^{(j)}(\vj_{k_j})\notag\\
 &= W_{k}^{(j)} \eigfm^{(j)}(\vj_{k_j}).
\end{align}
where the product of the weights in the third equality only depends upon the level $j$ and the vertex $\vj[1]_k$, and we have let
\begin{equation*}
	W_{k}^{(j)} := \prod_{n=1}^{j}\wj[n]_{k_n}
\end{equation*}
in the last equality.
By \eqref{eq:v1byvj},
\[
\begin{aligned}
\Phi(c,\vj[J]_k) &= \sum_{\ell=1}^{N} c_\ell \eigfm(\vj[J]_k)=\sum_{j=J_0}^{J} \sum_{\ell=N_{j-1}+1}^{N_j} c_\ell \eigfm(\vj[J]_k)\\
&=\sum_{j=J_0}^{J} \sum_{\ell=N_{j-1}+1}^{N_j}  c_\ell W_{k}^{(j)} \eigfm^{(j)}(\vj_{k_j})\\
&=\sum_{j=J_0}^{J} W_{k}^{(j)}\left(\sum_{\ell=N_{j-1}+1}^{N_j}  c_\ell \eigfm^{(j)}(\vj_{k_j})\right),
\end{aligned}
\]
thus completing the proof.
\end{proof}

\subsection{Computational Complexity Analysis}
Algorithm~\ref{alg:adft} gives the computational steps for evaluating $(\Phi^T f)_{\ell}$, $\ell=1,\ldots,N$ in Theorem~\ref{thm:adft}.  In the first step of Algorithm~\ref{alg:adft}, the total number of summations to compute all elements of Step~1 is no more than $\sum_{i=0}^{j-1} N_{i+1}$;
In the second step, the total number of multiplication and summation operations is at most $2\sum_{\ell=1}^{N}C=\mathcal{O}(N)$. Here $C$ is the constant which bounds the number of distinct values of the Haar basis (see Proposition~\ref{prop:spochaar}). Thus, the total computational cost of Algorithm~\ref{alg:adft} is $\mathcal{O}(N)$.

\medskip
\IncMargin{1em}
\begin{algorithm}[H]
\SetKwData{step}{Step}
\SetKwInOut{Input}{Input}\SetKwInOut{Output}{Output}
\BlankLine
\Input{A real-valued vector $f=(f_1,\dots,f_N)$ on the graph $\gph$; the Haar basis $\{\eigfm\}_{\ell=1}^{N}$ for $l_2(\gph)$ with the chain $\gph_{J\to J_0}$ and the associated basis $\{\eigfm^{(j)}\}_{\ell=1}^{N_{j}}$ for $l_2(\gph_j)$.}
\Output{The vector $\Phi^T f$ by adjoint Haar transform in \eqref{eq:adft} under the basis $\{\eigfm\}_{\ell=1}^{N}$.}
\begin{enumerate}
\item Evaluate the following sums for $j=J_0,\ldots,J-1$ in \eqref{eq:ws1} and \eqref{eq:wsj}.
\begin{equation*}
	\ws\bigl(f,\vj_k\bigr),\quad \vj_k \in V_j.
\end{equation*}
%by, for $\vj_k\in V_j$,
%\begin{equation*}
%	\ws\bigl(f,\vj_k\bigr):= \sum_{\vj[j+1]_{k'}\in \vj_k}\wj[j+1]_{k'} \ws[j+1]\bigl(f,\vj[j+1]_{k'}\bigr).
%\end{equation*}

%The summation $\ws[i]\bigl(f,\vj_k\bigr)$ is the weighted sum of the lower-level $\ws[i+1]\bigl(f,\vj[i+1]_{k'}\bigr)$.
\item
     For each $\ell$, let $j$ be the integer such that $N_{j-1}+1\leq\ell\leq N_j$, where $N_{J_0-1}:=0$. Evaluating
     $\sum_{k=1}^{N_j} \ws(f,\vj_k)\wj_k\eigfm^{(j)}(\vj_k)$
      in \eqref{eq:adftbyws} by the following two steps.\\[2mm]
(a)~Compute the product for all $\vj_k\in V_j$:\\[1mm]
$
	\hspace{0.7cm}T_\ell(f,\vj_k)=\ws(f,\vj_k)\wj_k\eigfm^{(j)}(\vj_k).
$\\[2mm]
(b)~Evaluate sum $\sum_{k=1}^{N_j}T_\ell(f,\vj_k)$.
\end{enumerate}
\caption{Fast Haar Transforms: Adjoint}\label{alg:adft}
\end{algorithm}
\medskip

By Theorem~\ref{thm:fdft}, the evaluation of the forward Haar transform $\Phi c$ can be implemented by Algorithm~\ref{alg:dft}.
In the first step of Algorithm~\ref{alg:dft}, the number of multiplications is no more than $\sum_{\ell=1}^{N}C=\mathcal{O}(N)$; in the second step, the number of summations is no more than $\sum_{\ell=1}^{N} C = \mathcal{O}(N)$; in the third step, the computational steps are $\mathcal{O}(N(\log N)^2)$; in the last step, the total number of summations and multiplications is $\mathcal{O}(N\log N)$.
Thus, the total computational cost of Algorithm~\ref{alg:dft} is $\mathcal{O}(N(\log N)^2)$.

Hence, Algorithms~\ref{alg:adft} and \ref{alg:dft} have linear computational cost (up to a $\log N$ term). We call these two algorithms \emph{fast Haar transforms (FHTs)} under Haar basis on the graph.

\begin{proposition}\label{prop:signalrecover} The adjoint and forward Haar Transforms in Algorithms~\ref{alg:adft} and \ref{alg:dft} are invertible in that for any vector $f$ on graph $\gph$,
\begin{equation*}
	f = \Phi (\Phi^T f).
\end{equation*}
\end{proposition}
Proposition~\ref{prop:signalrecover} shows that the forward Haar transform can recover graph data $f$ from the adjoint Haar transform $\Phi^T f$. This means that forward and adjoint Haar transforms have zero-loss in graph data transmission.

Haar convolution, which computational strategy given by Algorithm~\ref{alg:haarconv}, can be evaluated fast by FHTs in Algorithms~\ref{alg:adft} and \ref{alg:dft}. From the above discussion, the total computational cost of Algorithm~\ref{alg:haarconv} is $\mathcal{O}(N(\log N)^2)$. That is, using FHTs, Haar convolution can be evaluated in near linear computational complexity.

\section{Graph neural networks with Haar transforms}\label{sec:hanet}
\subsection{Models}
The Haar convolution in \eqref{defn:graphconv} can be applied to any architecture of graph neural network. For graph classification and graph-based regression tasks, we use the model with convolutional layer consisting of $m$-hidden neutrons and a non-linear activation function $\sigma$ (e.g. ReLU): for $i=1,2\ldots,m$,

\medskip
\IncMargin{1em}
\begin{algorithm}[H]
\caption{Fast Haar Transforms: Forward}
   \label{alg:dft}
\SetKwData{step}{Step}
\SetKwInOut{Input}{Input}\SetKwInOut{Output}{Output}
\BlankLine
\Input{A real-valued vector $c=(c_1,\dots,c_N)$ on graph $\gph$; the Haar basis $\{\eigfm\}_{\ell=1}^{N}$ for $l_2(\gph)$ associated with the chain $\gph_{J\to J_0}$ and the associated orthonormal basis $\{\eigfm^{(j)}\}_{\ell=1}^{N_{j}}$ for $l_2(\gph_j)$.}
\Output{The vector $\Phi c$ by forward Haar transform in \eqref{eq:dft} under the basis $\{\eigfm\}_{\ell=1}^{N}$.}
\begin{enumerate}
\item For each $\ell$, let $j$ be the integer such that $N_{j-1}+1\leq\ell\leq N_j$, where $N_{J_0-1}:=0$.\\
For all $k=1,\dots,N_j$, compute the product\\[1mm]
\hspace{0.7cm}$
	t_\ell(c,\vj_k):=c_\ell \eigfm^{(j)}(\vj_{k}).
$
\item For each $j=J_0,\ldots,J$, evaluate the sums\\[1mm]
\hspace{0.7cm}$
		s(c,\vj_{k_j}):=\sum_{\ell=N_{j-1}+1}^{N_j}t_\ell(c,\vj_{k_j}).
$
\item Compute the $W_{k}^{(j)}$ for $k=1,\dots,N$ and $j=J_0,\dots,J-1$ by \eqref{eq:Wkj}.
\item Compute the weighted sum\\[1mm]
\hspace{0.7cm}$
		(\Phi c)_k =\sum_{j=J_0}^{J}W_{k}^{(j)} s(c,\vj_{k_j}), \quad k=1,\dots,N.
$
\end{enumerate}
\end{algorithm}
\medskip
\IncMargin{1em}
\begin{algorithm}[H]
\caption{Fast Haar Convolution}\label{alg:haarconv}
\SetKwData{step}{Step}
\SetKwInOut{Input}{Input}\SetKwInOut{Output}{Output}
\BlankLine   
\Input{Real-valued vectors $g:=(g_1,\dots,g_N)$ and $f:=(f_1,\dots,f_N)$ on $\gph$; chain $\gph_{J_0\to J}$ of graph $\gph$ where $\gph_{J}:=\gph$.}
\Output{Haar convolution $g\star f$ of $g$ and $f$ as given by Definition~\ref{defn:graphconv}.}
\begin{enumerate}
	\item Compute the adjoint Haar transforms $\Phi^T g$ and $\Phi^T f$ by Algorithm~\ref{alg:adft}.
	\item Compute the point-wise product of $\Phi^T g$ and $\Phi^T f$.
	\item Compute the forward Haar transform of $(\Phi^T g)\odot (\Phi^T f)$ by Algorithm~\ref{alg:dft}.
\end{enumerate}
\end{algorithm}
\begin{align}\label{eq:gnn}
	f^{\rm out}_{i} &=\sigma \left(\sum_{j=1}^{d}\Phi \bigl(g_{i,j}\odot (\Phi^T f^{\rm in}_{j})\bigr)\right)\notag\\
	&=\sigma \left(\sum_{j=1}^{d}\Phi G_{i,j} \Phi^T f^{\rm in}_{j}\right),
\end{align}
for input graph data $F^{\rm in}=(f^{\rm in}_1,f^{\rm in}_2,\ldots,f^{\rm in}_d)\in R^{N\times d}$ with $N$ nodes and $d$ input features (for each vertex). Here, the feature $f^{\rm in}_j$ of the input graph data is convolved with the learnable filter $g_{i,j}\in \R^N$ by Haar transforms, and then all Haar-transformed features are fused as a new feature $f^{\rm out}_i$. This gives the output matrix $F^{\rm out}=(f^{\rm out}_1,f^{\rm out}_2,\ldots,f^{\rm out}_m)\in \R^{N\times m}$. If we write $G_{i,j}\in \R^{N\times N}$ as the diagonal matrix of filter $g_{i,j}$, the convolutional layer has the compact form of the second equality in \eqref{eq:gnn}.
We call the GNN with Haar convolution in \eqref{eq:gnn} \emph{HANet}.

\textbf{Weight detaching.} For each layer, $O(Ndm)$ parameters need to be tuned. To reduce the number of parameters, we can replace the filter matrix $G_{i,j}$ by a unified diagonal filter matrix $G$ and a compression matrix $W\in R^{d\times m}$ (which is a detaching approach used in conventional CNN for extracting features). This then leads to a concise form
\begin{equation}\label{gnn_haar_eq}
{F^{\rm out}}= \sigma\left(\Phi\big(G(\Phi^TF^{\rm in})\big)W\right).
\end{equation}
Then, it requires $\mathcal{O}(N+dm)$ parameters to train.
Recall that constructing the Haar basis uses a chain $\gph_{J\to J_0}$ for the graph $\gph$, one can implement weight sharing based on the same chain structure. Specifically, one can use $k$-means clustering algorithm \cite{Lloyd1982} or METIS algorithm \cite{KaKu1998} to generate a chain, which captures clustering  feature of the graph. Suppose a coarser level $J_1$ ($J_0\leq J_1<J$) having $K$ clusters, then all vertices in the same cluster share the common filter parameter. The corresponding children vertices in level $J_1-1$ share the same filter parameters as used in their parent vertices, and the bottom level corresponds to the whole set of vertices of the input graph. Thus, the number of parameters is reduced to $\mathcal{O}(K+dm)$.

The HANet uses $d$ times fast Haar convolutions (consisting of $d$-times adjoint and forward Haar transforms). The computational cost of Haar convolution in HANet is then $\mathcal{O}(N(\log N)^2d)$. Deep GNNs with Haar convolution are built by stacking up multiple Haar convolutional layers of \eqref{gnn_haar_eq}, followed by an output layer.
\begin{figure}[ht]
\vskip 0.2in
\centering
\begin{minipage}{1\textwidth}
\includegraphics[width=\columnwidth]{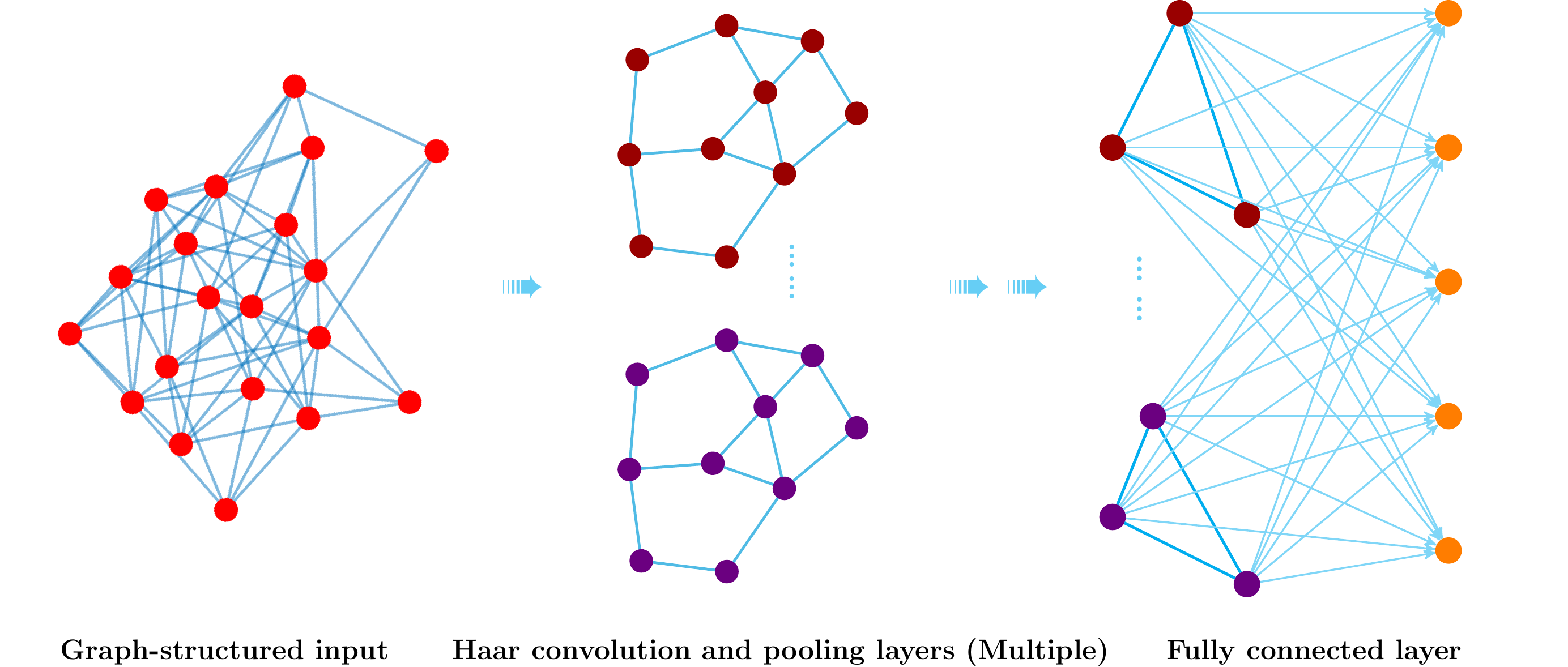}
\end{minipage}
\vskip 2mm
\begin{minipage}{0.85\textwidth}
\centering
\caption{Network architecture of HANet with multiple Haar convolutional layers and then fully connected by softmax.}\label{fig:topohanet}
\end{minipage}
\end{figure}

\begin{figure}
\centering
\begin{minipage}{0.8\columnwidth}
\includegraphics[width=\columnwidth]{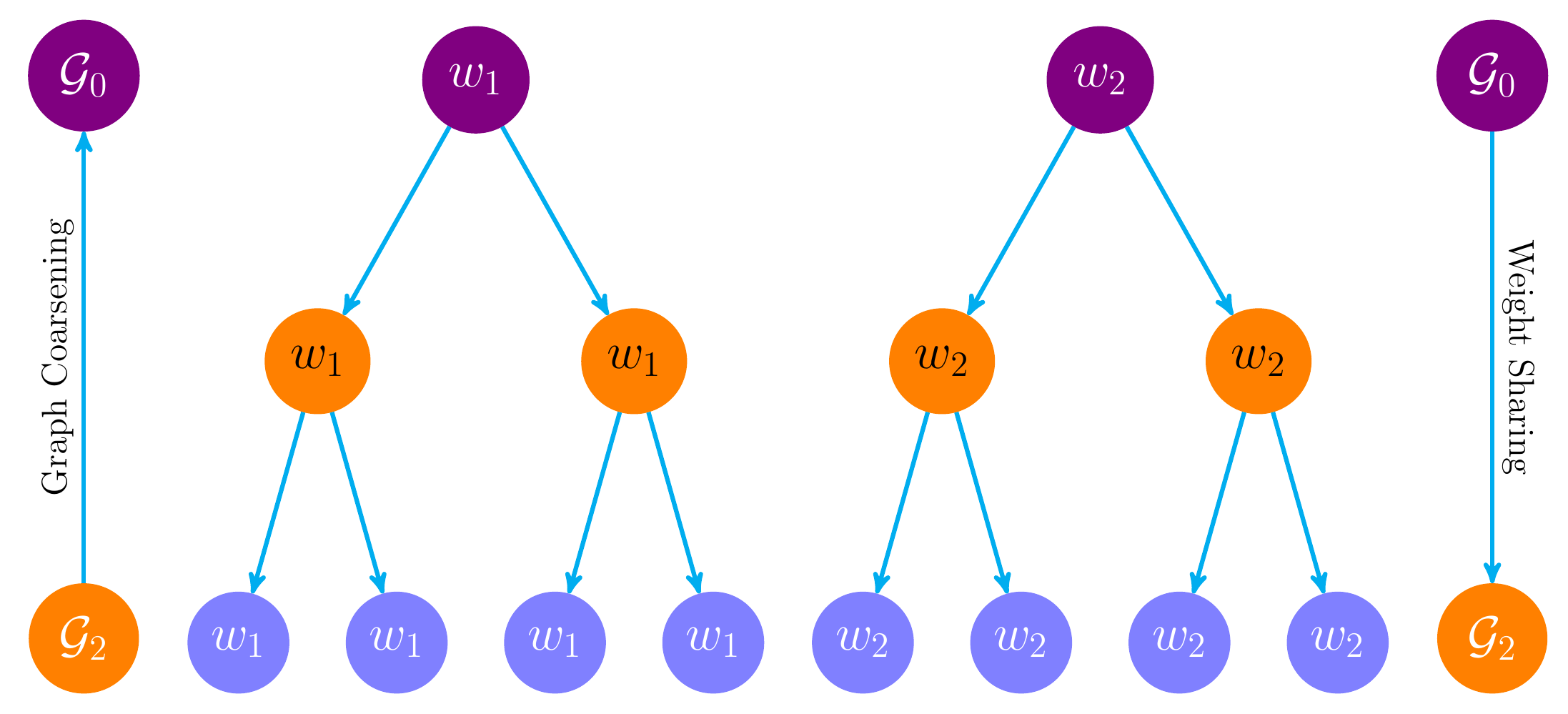}
\end{minipage}
\vskip 2mm
\begin{minipage}{0.85\textwidth}
\centering
\caption{Weight sharing for Haar convolution and graph coarsening for graph pooling for the chain $\gph_{2\to 0}$.}\label{fig:weightshare_graphcoarsen}
\end{minipage}
\end{figure}

\textbf{HANet for graph classification and graph-based regression.} These two tasks can be formulated as the following supervised learning: Given a collection of graph-structured data $\{f_i\}_{i=1}^{n}$ with labels $\{y_i\}_{i=1}^{n}$, the classification task is to find a mapping that can classify or regress labels.
The model of HANet uses a similar architecture of canonical deep convolutional neural network: it has several convolutional layers and fully connected dense layers but the convolutional layer uses Haar convolution.
Figure~\ref{fig:topohanet}a shows the flowchat for the architecture of HANet with multiple Haar convolutional layers:
the chain $\gph_{J\to J_0}$ and the Haar basis $\eigfm$ and the associated basis $\eigfm^{(j)}$, $j=J_0,\dots,J$ are pre-computed; graph-structured input $f$ is Haar-convoluted with filter $g$ which is of length $N_J$ but with $N_{J-1}$ independent parameters, where $g$ is expanded from level $J-1$ to $J$ by weight sharing, and the output $f^{\rm out}$ of the first layer is the ReLU of the Haar convolution of $g$ and $f$; the graph pooling reduces $f^{\rm out}$ of size $N_{J}$ to $\widetilde{f}^{\rm in}$ of size $N_{J-1}$; and in the second Haar convolutional layer, the input is $\widetilde{f}^{\rm in}$ and the Haar basis is $\eigfm^{(J-1)}$; the following layers continue this process; the final Haar convolutional layer is fully connected by one or multiple dense layers. For classification, an additional dense layer with softmax function is used.

\textbf{HANet for node classification.} In node classification, the whole graph is the only single input data, where a fractional proportion of nodes are labeled. The output is the whole graph with all unknown labels predicted. Here we use the following GNN with two layers.
\begin{equation}\label{eq:gnn1}
	\hbox{HANet}(f^{\rm in})
	:= {\rm softmax}\Bigl({\rm HC}^{(2)}\bigl({\rm ReLU}\bigl({\rm HC}^{(1)}\bigl(f^{\rm in}\bigr)\bigr)\bigr)\Bigr)
\end{equation}
where ${\rm HC}^{(1)}$ and ${\rm HC}^{(2)}$ are the Haar convolutional layers
\begin{equation*}
	 {\rm HC}^{(i)}(f) := \widehat{A}(w_1^{(i)}\star f) w_2^{(i)},\quad i=1,2,
\end{equation*}
where we use the modified Haar convolution $w_1^{(i)}\star f=\Phi\bigl(w_1^{(i)}\odot(\Phi^{T}f)\bigr)$.
For a graph with $N$ nodes and $M$ features, in the first Haar convolutional layer, the filter $w^{(1)}_1$ contains $N_0\times M$ parameters and is extended to a matrix $N\times M$ by weight sharing, where $N_0$ is the number of nodes at the coarsest level. The $w^{(1)}_2$
%which is a matrix with size $\#{\rm neurons}\times M$
plays the role of weight compression and feature extraction. The first layer is activated by the rectifier and the second layer is fully connected with softmax.
The $\widehat{A}$, which is defined in \cite{KiWe2017}, is the square matrix of size $N$ determined by the adjacency matrix of the input graph. This smoothing operation compensates the information loss in coarsening by taking a weighted average of features of each vertex and its neighbours. For vertices that are densely connected, it makes their features more similar and significantly improves the ease of node classification task \cite{LiHaWu2018}.

\subsection{Technical Components}\label{sec:technical}
\textbf{Fast computation for HANet.} Complexity analysis of FHTs above shows that HANet is more efficient than GNNs with graph Fourier basis. The graph convolution of the latter incurs $\mathcal{O}(N^3)$ computational cost. Many methods have been proposed to improve the computational performance for graph convolution. For example, ChebNet \cite{DeBrVa2016} and GCN \cite{KiWe2017} use localized polynomial approximation for the spectral filters; GWNN \cite{GWNN} constructs sparse and localized graph wavelet basis matrix for graph convolution. These methods implement the multiplication between a sparse matrix (e.g. the refined adjacency matrix $\hat{A}$ in GCN or the wavelet basis matrix $\psi_{s}$ in GWNN \cite{GWNN}) and input matrix $F$ in the convolutional layer. However, to compute either $\hat{A}F$ or $\psi_{s}F$, the computational complexity, which is roughly proportional to $\mathcal{O}(\varepsilon N^2d)$, to a great extent relies on the sparse degree of $\hat{A}$ or $\psi_{s}$, where $\varepsilon$, $\varepsilon\in[0,1]$, represents the percentage of non-zero elements in a square matrix. The $\mathcal{O}(\varepsilon N^2d)$ may be significantly higher than $\mathcal{O}(N(\log N)^2d)$ as long as $\varepsilon$ is not extremely small, indicating that our FHTs outperform these methods especially when $N$ is quite large and $\varepsilon\approx1$. In addition, the fast computation for sparse matrix multiplication (see \cite{GoVa2012}) can further speed up the evaluation of Haar convolution. HANet with sparse FHTs can be developed by using the strategy in \cite{HaInKaPr2012simple,InKaPr2014}.

\textbf{Chain.} In HANet, the chain and the Haar basis can be pre-computed since the graph structure is already known. In particular, the chain is computed by a modified version of the METIS algorithm \cite{KaKu1998}, which fast generates a chain for the weight matrix of a graph. In many cases, the parents of a chain from METIS have at least two children, and then the weighted chain is a filtration and Proposition~\ref{prop:spochaar} applies.

\textbf{Weight sharing for filter.} In the HANet, one can use weight sharing given in Section~\ref{sec:haarconv} for filters. By doing this, we exploit the local topological property of the graph-structured data to extract the common feature of neighbour nodes and meanwhile reduce the independent parameters of the filter. Weight sharing can be added in each convolutional layer of HANet. For chain $\gph_{J\to J_0}$ with which the Haar basis is associated, weight sharing can act from the coarsest level $J_0$ to the finest level $J$ or from any level coarser than $J$ to $J$. For a filtration, the weight sharing shrinks the number of parameters by at least rate $2^{-(J-J_0)}$, see Figure~\ref{fig:topohanet}b.

%We will demonstrate more simulation results to highlight this technical point.
\textbf{Graph pooling.} We use max graph pooling between two convolutional layers of HANet. Each pooled input is the maximum over children nodes of each node of the current layer of the chain. The pooling uses the same chain as the Haar basis at the same layer. For example, after pooling, the second layer uses the chain $\gph_{(J-1)\to J_0}$, as illustrated in Figure~\ref{fig:topohanet}. By the construction of Haar basis in Section~\ref{subsec:haarbasis}, the new Haar basis associated with $\gph_{(J-1)\to J_0}$ is exactly the pre-computed basis $\{\eigfm^{(J-1)}\}_{\ell=1}^{N_{J-1}}$.

%%%
\section{Experiments}\label{sec:experiments}
In this section, we test the proposed HANet on Quantum Chemistry (graph-based regression) and Citation Networks (node classification). The experiments for graph classification were carried out under the Google Colab environment with Tesla K80 GPU while for node classification were under the UNIX environment with a 3.3GHz Intel Core i7 CPU and 16GB RAM. All the methods were implemented in TensorFlow. SGD+Momentum and Adam optimization methods were used in the experiments.

%\subsection{MNIST for Graph Signal Classification}
%We treat the MNIST digits classification as a learning problem on graphs, following \cite{DeBrVa2016,Monti_etal2017}, where each pixel of an image with resolution $28\times28$ is one of the $784$ vertices of the graph. Of $70,000$ images, we use $55,000$ for training, $5,000$ for validation and the remaining $10,000$ for test. We compare HANet against GNN with graph Laplacian \cite{BrZaSzLe2013} and ChebNet \cite{DeBrVa2016}, all with LeNet-5-like architecture \cite{LeBoBeHa1998}. The edges are determined by the spatial relation between vertices. A vertex has an edge to the $8$ neighbours and has no edge to other vertices.
%The task is to recognize an image of hand-written digit as one of the ten digits $0,1,\dots,9$.
%We use similar hyper-parameter setting from ChebNet\footnote{\url{https://github.com/mdeff/cnn_graph}}: dropout probability $0.5$, regularization weight $2\times10^{-4}$, initial learning rate $0.02$, learning rate decay $0.95$ and momentum $0.9$.
%The chain for Haar basis is $\gph_{5\to0}$. The finest level $\gph_{5}=\gph$ has $784$ nodes and the graphs at the following levels have $218$, $64$, $18$, $6$ and $3$ nodes. For weight sharing and graph pooling we use the chain $\gph_{(J-k+1)\to J-k}$ for the $k$th convolutional layer. The test accuracy of HANet for MNIST is $98.60\%$, which is close to $99.17\%$ of ChebNet and higher than $96.26\%$ of the GNN with graph Laplacian. This shows that HANet is able to learn complicated big graph data in graph signal classification where the Haar convolution preserves the geometric information of data.

\subsection{Quantum Chemistry for Graph-based Regression}
We test HANet on QM7 \cite{BlRe2009,RuTkMuLi2012}, which contains $7165$ molecules. Each molecule is represented by the Coulomb (energy) matrix and its atomization energy. We treat each molecule as a weighted graph where the nodes are the atoms and the adjacency matrix is the $23\times23$-Coulomb matrix of the molecule, where the true number of atoms may be less than $23$. The atomization energy of the molecule is the label. As in most cases the adjacency matrix is not fully ranked, we take the average of the Coulomb matrices of all molecules as the common adjacency matrix, for which we generate the Haar basis. To avoid exploding gradients in parameter optimization, we take the standard score of each entry over all Coulomb matrices as input.
\begin{table}[ht]
\caption{Test mean absolute error (MAE) comparison on QM7}\label{tab:qm7_results}
\vspace{-5mm}
\begin{center}
\begin{small}
%\begin{sc}
\begin{tabular}{c|c}
\toprule
Method & Test MAE  \\
\midrule
  RF \cite{Breiman2001} & $122.7\pm4.2$ \\
  Multitask \cite{Ramsundar_etal2015} & $123.7\pm15.6$ \\
  KRR \cite{CoVa1995}  & $110.3\pm4.7$ \\
  GC \cite{AlRaPaPa2017} & $77.9\pm2.1$ \\
  Multitask(CM) \cite{Wu_etal2018}  & $10.8\pm1.3$ \\
  KRR(CM) \cite{Wu_etal2018} & $10.2\pm0.3$ \\
  DTNN \cite{Schutt_etal2017}& $8.8\pm3.5$ \\
  ANI-1 \cite{SmIsRo2017}& $2.86\pm0.25$\\
  \midrule
  \textbf{HANet} & $\boldsymbol{9.50\pm0.71}$ \\
\bottomrule
\end{tabular}
%\end{sc}
\end{small}
\end{center}
%\vskip -0.2in
\end{table}

%\begin{table}[ht]
%\caption{Test mean absolute error (MAE) comparison on QM7}\label{tab:qm7_results}
%%\vskip 0.15in
%\begin{center}
%\begin{small}
%%\begin{sc}
%\begin{tabular}{c|c}
%\toprule
%Method & Test MAE  \\
%\midrule
%  RF & $122.7\pm4.2$ \\
%  Multitask & $123.7\pm15.6$ \\
%  KRR  & $110.3\pm4.7$ \\
%  GC  & $77.9\pm2.1$ \\
%  Multitask(CM)	  & $10.8\pm1.3$ \\
%  KRR(CM) & $10.2\pm0.3$ \\
%  DTNN & $8.8\pm3.5$ \\
%  ANI-1 & $2.86\pm0.25$\\
%  \midrule
%  \textbf{HANet} & $\boldsymbol{9.50\pm0.71}$ \\
%\bottomrule
%\end{tabular}
%%\end{sc}
%\end{small}
%\end{center}
%%\vskip -0.2in
%\end{table}

\begin{figure*}[ht]
%\vskip 0.1in
\centering
\begin{minipage}{\columnwidth}
\begin{minipage}{0.32\columnwidth}
\centering
\vspace{-1mm}
	\includegraphics[width=0.8\columnwidth]{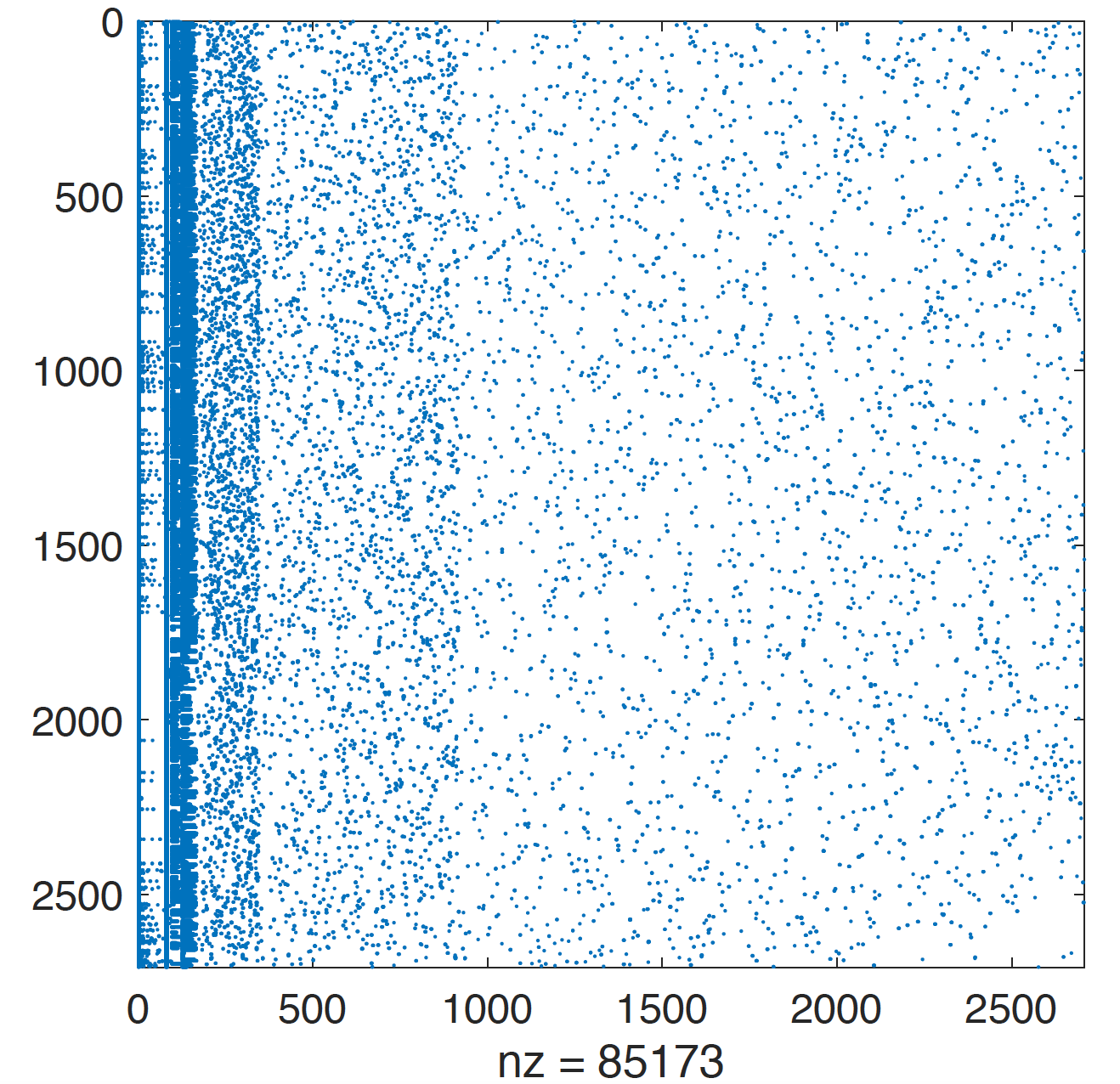}\\[-2mm]
	{\scriptsize(a) Haar transform matrix $\Phi$}
\end{minipage}
	\hskip -0.05in
\begin{minipage}{0.32\columnwidth}
\centering
	\includegraphics[width=\columnwidth]{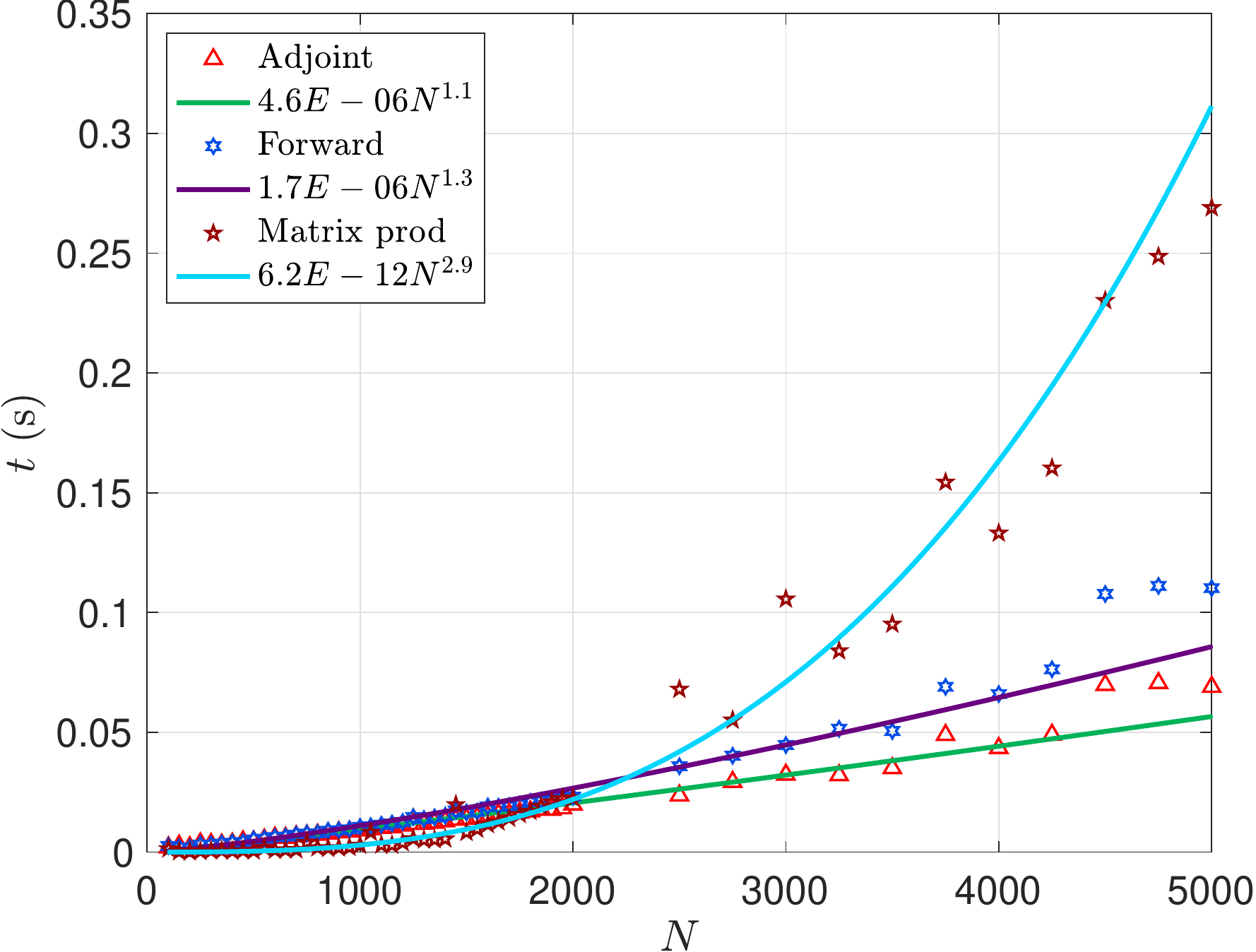}\\[-2mm]
	{\scriptsize(b) CPU time of FFTs}
\end{minipage}
\hskip 0.1in
\begin{minipage}{0.32\columnwidth}
\centering
	\includegraphics[width=0.98\columnwidth]{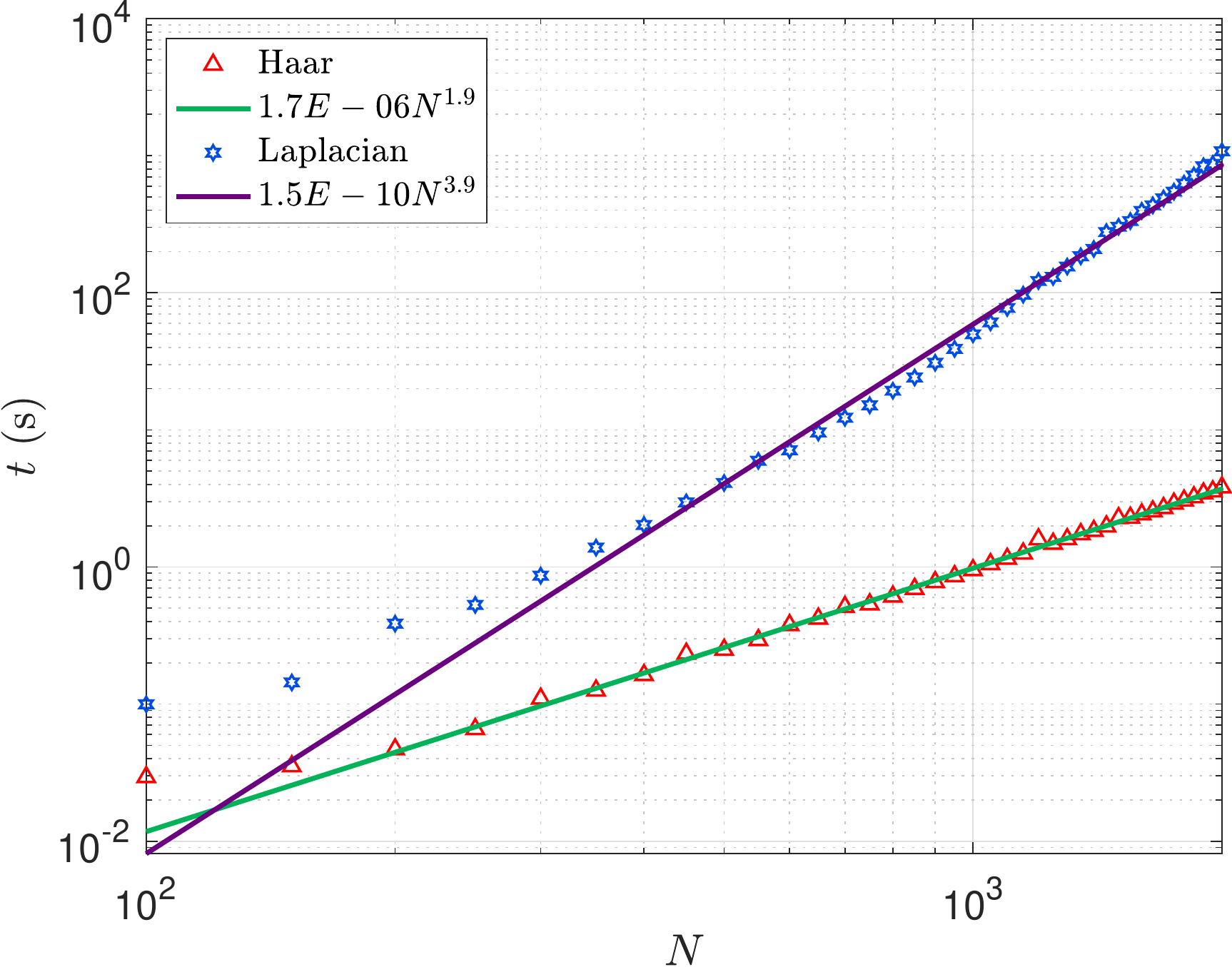}\\[-2mm]
	{\scriptsize(c) CPU time of generating basis}
\end{minipage}
\end{minipage}
\vskip 2mm
\begin{minipage}{0.85\textwidth}
\centering
\caption{(a) Haar basis $\Phi$ for Cora with a chain of $11$ levels (by METIS). Each column is a vector of the Haar basis. The sparsity of the Haar transform matrix is $98.84\%$ (i.e. the proportion of zero entries). (b) Comparison of CPU time for FHTs and Direct Matrix Product for the Haar basis for graphs with nodes $\leq 5,000$. (c) Comparison of CPU time for generating the orthonormal bases for Haar and graph Laplacian on graphs with nodes $\leq 2,000$.}
\label{fig:fft}
\end{minipage}
%\vskip -0.2in
\end{figure*}

\begin{table*}[htbp!]
\centering
\begin{minipage}{0.85\textwidth}
\centering
\caption{Sparsity of Haar basis and CPU time for basis generating, adjoint FHT (AFHT) and forward FHT (FFHT) on citation networks data set}\label{tab:citationhaar}
\end{minipage}
\vspace{-5mm}
\begin{center}
\begin{small}
%\begin{sc}
\begin{tabular}{c|c|c|c|c|c}
\toprule
Dataset & Basis Size & Sparsity & Generating Time (s) & AFHT Time (s) & FFHT Time (s) \\
\midrule
  Citeseer  & $3327$ & $99.58\%$ & $1.93509$ & $0.05276$ & $0.05450$ \\
  Cora  & $2708$ & $98.84\%$ & $0.86429$ & $0.06908$ & $0.05515$ \\
  Pubmed  & $19717$ & $99.84\%$ & $62.67185$ & $1.08775$ & $1.55694$ \\
\bottomrule
\end{tabular}
%\end{sc}
\end{small}
\end{center}
%\vskip -0.1in
\end{table*}

The network architecture of HANet contains $2$ layers of Haar convolution with $8$ and $2$ filters and then $2$ fully connected layers with $400$ and $100$ neurons. As the graph is not big, we do not use graph pooling or weight sharing. Following \cite{Gilmer_etal2017}, we use mean squared error (MSE) plus $\ell_2$ regularization as the loss function in training and mean absolute error (MAE) as the test metric. We repeat the experiment over $5$ splits with the same proportion of training and test data but  with different random seeds. We report the average performance and standard deviation for the HANet in Table~\ref{alg:adft} compared against other public results \cite{Wu_etal2018} by methods Random Forest (RF) \cite{Breiman2001}, Multitask Networks (Multitask) \cite{Ramsundar_etal2015}, Kernel Ridge Regression (KRR) \cite{CoVa1995}, Graph Convolutional models (GC) \cite{AlRaPaPa2017}, Deep Tensor Neural Network (DTNN) \cite{Schutt_etal2017}, ANI-1 \cite{SmIsRo2017}, KRR and Multitask with Coulomb Matrix featurization (KRR(CM)/Multitask(CM)) \cite{Wu_etal2018}. It shows that HANet ranks third in the list with average test MAE $9.50$ and average relative MAE $4.31\times10^{-6}$, which offers a good approximator for QM7 regression.

\subsection{Citation Networks for Node Classification}
%%%%%%%%% Table III %%%%%%%%
\begin{table}[htbp!]
\caption{Test accuracy comparison on citation networks}\label{tab:Com_results}
\vspace{-5mm}
\begin{center}
\begin{small}
%\begin{sc}
\begin{tabular}{c|c|c|c}
\toprule
Method & Citeseer & Cora & Pubmed  \\
\midrule
MLP \cite{KiWe2017}& $55.1$ & $46.5$ & $71.4$\\
  ManiReg \cite{BeNiSi2006} & $60.1$ & $59.5$ & $70.7$  \\
  SemiEmb \cite{WeFrMoCo2012} & $59.6$ & $59.0$ & $71.1$  \\
  LP \cite{ZhGhLa2003} & $45.3$ & $68.0$ & $63.0$  \\
  DeepWalk \cite{Perozzi_etal2014}& $43.2$ & $67.2$ & $65.3$\\
  ICA \cite{LuGe2003}& 69.1 & $75.1$ & $73.9$\\
  Planetoid \cite{YaCoSa2016}& $64.7$ & $75.7$ & $77.2$ \\
  ChebNet \cite{DeBrVa2016}& $69.8$ & $81.2$ & $74.4$\\
  GCN \cite{KiWe2017} & $70.3$ & $81.5$ & $79.0$  \\
  \midrule
  \textbf{HANet} & $\boldsymbol{70.1}$ & $\boldsymbol{81.9}$ & $\boldsymbol{79.3}$  \\
\bottomrule
\end{tabular}
%\end{sc}
\end{small}
\end{center}
%\vskip -0.2in
\end{table}

%%%%%%%%% Table I %%%%%%%%
%\begin{table}[htbp!]
%\caption{Test accuracy comparison on citation networks}\label{tab:Com_results}
%%\vskip 0.15in
%\begin{center}
%\begin{small}
%%\begin{sc}
%\begin{tabular}{c|c|c|c}
%\toprule
%Method & Citeseer & Cora & Pubmed  \\
%\midrule
%MLP & $55.1$ & $46.5$ & $71.4$\\
%  ManiReg  & $60.1$ & $59.5$ & $70.7$  \\
%  SemiEmb  & $59.6$ & $59.0$ & $71.1$  \\
%  LP  & $45.3$ & $68.0$ & $63.0$  \\
%  DeepWalk & $43.2$ & $67.2$ & $65.3$\\
%  ICA & 69.1 & $75.1$ & $73.9$\\
%  Planetoid & $64.7$ & $75.7$ & $77.2$ \\
%  ChebNet & $69.8$ & $81.2$ & $74.4$\\
%  GCN  & $70.3$ & $81.5$ & $79.0$  \\
%  \midrule
%  \textbf{HANet} & $\boldsymbol{70.1}$ & $\boldsymbol{81.9}$ & $\boldsymbol{79.3}$  \\
%\bottomrule
%\end{tabular}
%%\end{sc}
%\end{small}
%\end{center}
%%\vskip -0.2in
%\end{table}

We test the model in \eqref{eq:gnn1} on citation networks Citeseer, Cora and Pubmed~\cite{Sen_etal2008}, following the experimental setup of \cite{YaCoSa2016,KiWe2017}. The Citeseer, Cora and Pubmed are $6$, $7$ and $3$ classification problems with nodes $3327$, $2708$ and $19717$, edges $4732$, $5429$ and $44338$, features $3703$, $1433$ and $500$, and label rates $0.036$, $0.052$ and $0.003$ respectively. In Table~\ref{tab:Com_results}, we compare the performance of the model \eqref{eq:gnn1} of HANet with methods Multilayer Perceptron (MLP), Manifold Regularization (ManiReg)~\cite{BeNiSi2006}, Semi-supervised Embedding (SemiEmb) \cite{WeFrMoCo2012}, Traditional Label Propagation (LP) \cite{ZhGhLa2003}, DeepWalk \cite{Perozzi_etal2014}, Link-based Classification (ICA) \cite{LuGe2003}, Planetoid \cite{YaCoSa2016}, ChebNet \cite{DeBrVa2016} and GCN~\cite{KiWe2017}. We repeat the experiment $10$ times with different random seeds and report the average test accuracy of HANet. As shown in Table~\ref{tab:Com_results}, HANet has the top test accuracies on Cora and Pubmed and ranks second on Citeseer.

\medskip
\begin{figure}[h]
%\vskip 0.08in
\centering
\begin{minipage}{0.7\columnwidth}
%\hskip 0.1in
\begin{minipage}{0.92\columnwidth}
	\includegraphics[width=\columnwidth]{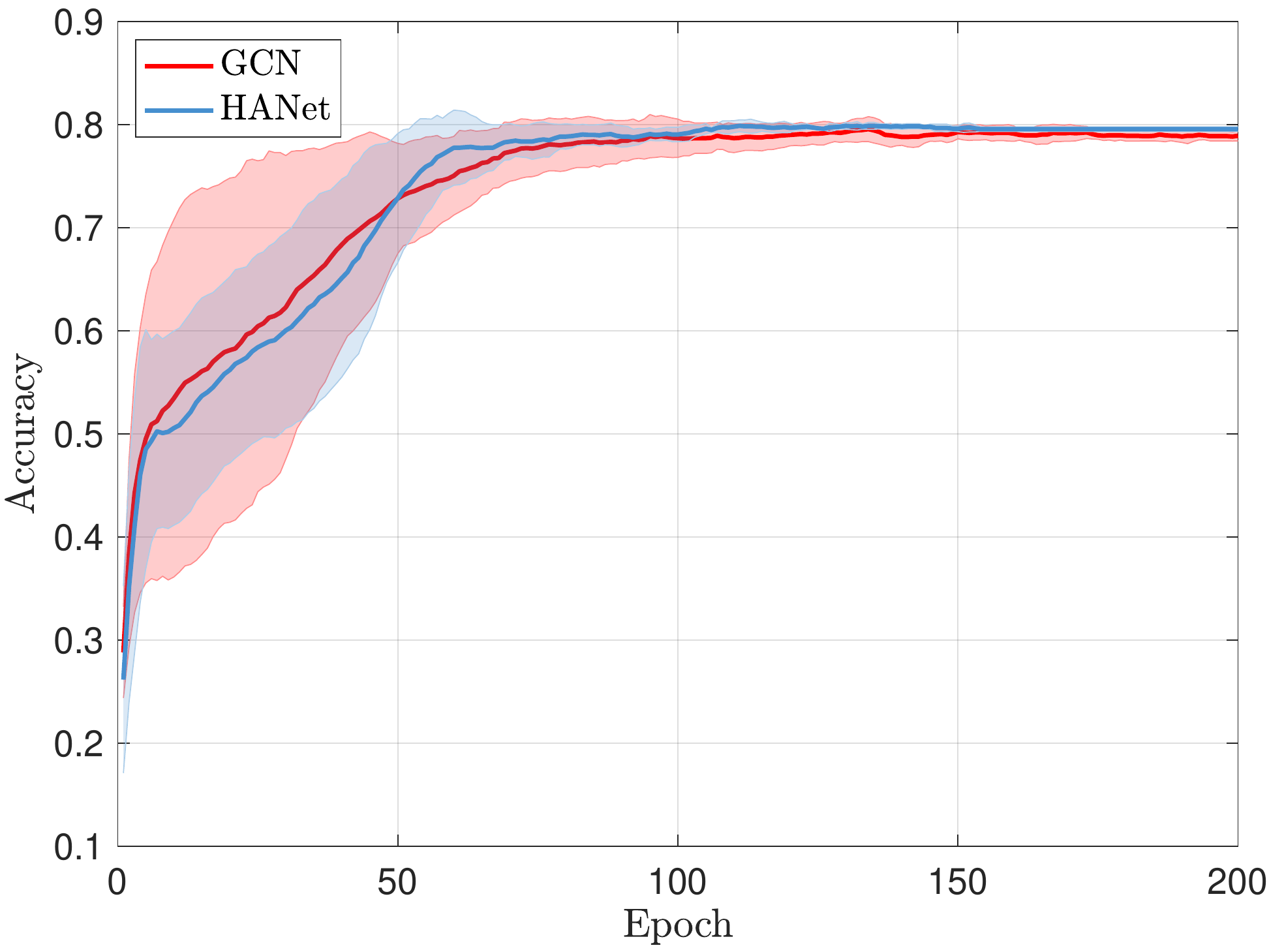}
\end{minipage}
	\hskip -0.5\columnwidth
\begin{minipage}{0.46\columnwidth}
\vskip 5mm
	\includegraphics[width=\columnwidth]{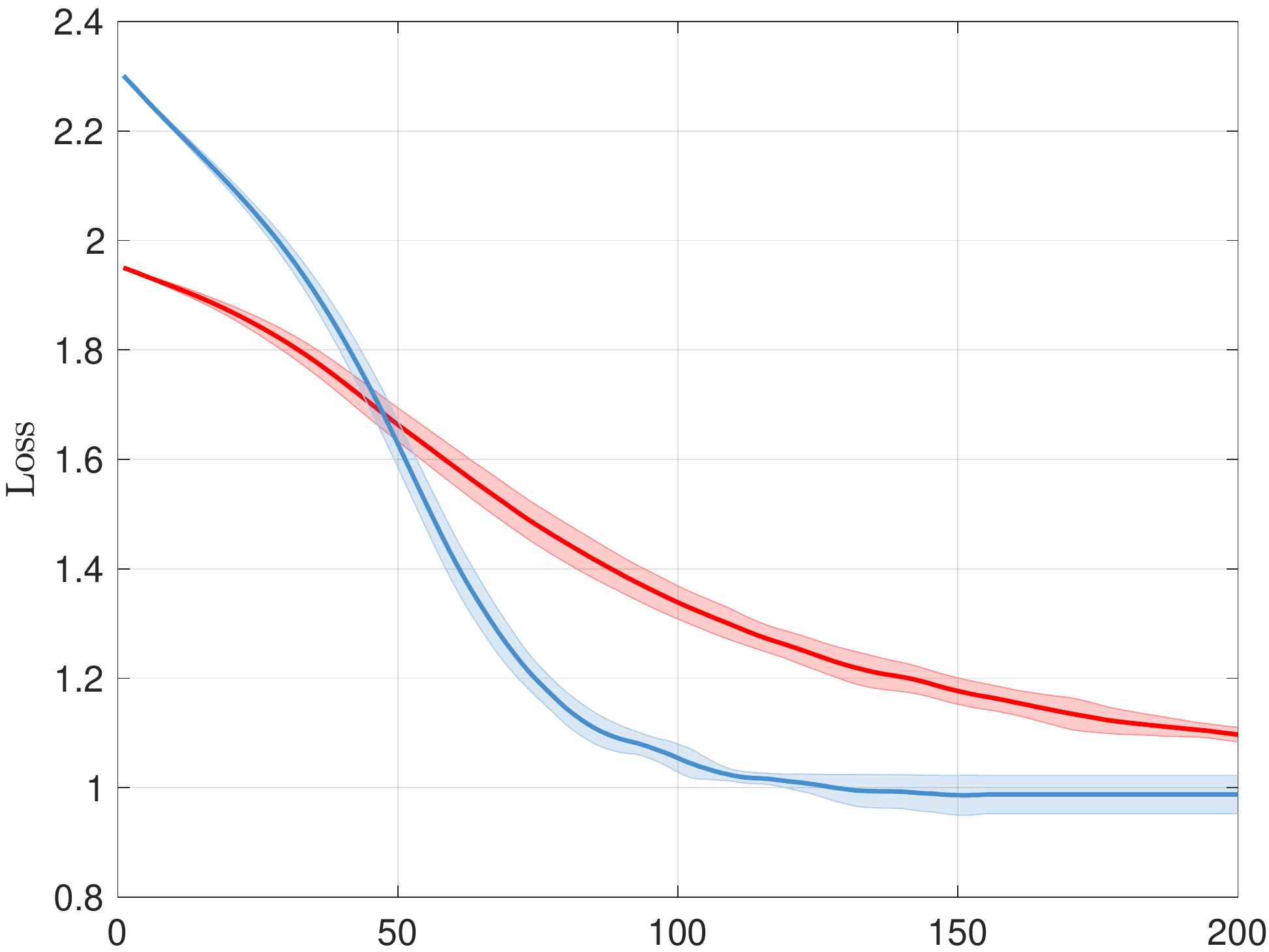}
\end{minipage}
\end{minipage}
%\vskip -0.1in
\vskip 2mm
\begin{minipage}{0.85\textwidth}
\centering
\caption{Main figure: Mean and standard deviation of validation accuracies of HANet and GCN on Cora with epoch $\leq 200$. Figure in the lower right corner: Validation loss function of HANet and GCN.}
\label{fig:meanstd.hanetgcn}
\end{minipage}
\end{figure}
Figure~\ref{fig:meanstd.hanetgcn} shows the mean and standard deviation of validation accuracies and the validation loss with up to epoch $200$ of HANet and GCN. HANet achieves slightly higher max accuracy as well as smaller standard deviation, and the loss also converges faster than GCN.

\subsection{Haar Basis and FHTs}\label{subsec:experimentFHT}
In Figure~\ref{fig:fft}a, we show the matrix of the Haar basis vectors for Cora, which has sparsity (i.e. the proportion of zero entries) $98.84\%$. The associated chain $\gph_{10\to0}$ has $2708$, $928$, $352$, $172$, $83$, $41$, $20$, $10$, $5$, $2$, $1$ nodes from level $10$ to $0$. Figure~\ref{fig:fft}b shows the comparison of time for FHTs with direct matrix product. It illustrates that FHTs have nearly linear computational cost while the cost of matrix product grows at $\mathcal{O}(N^3)$ for a graph of size $N$. Figure~\ref{fig:fft}c shows the comparison of time for generating the Haar basis and the basis for graph Laplacian: Haar basis needs significantly less time than that for graph Laplacian. Table~\ref{tab:citationhaar} gives the sparsity (i.e. the proportion of zero entries) and the CPU time for generating Haar basis and FHTs on three datasets.
All sparsity values for three datasets are very high (around $99\%$), and the computational cost of FHTs is proportional to $N$.

\section{Conclusion}
We introduce Haar basis and Haar transforms on a coarse-grained chain on the graph. From Haar transforms, we define Haar convolution for GNNs, which has a fast implementation in view of the sparsity of the Haar transform matrix. Haar convolution gives a sparse representation of graph data and captures the geometric property of the graph data, and thus provides an effective graph convolution for any architecture of GNN.

\section*{Acknowledgements}
Ming Li acknowledges support by the National Natural Science Foundation of China under Grant 61802132 and Grant 61877020, and the China Post-Doctoral Science Foundation under Grant 2019T120737.
Yu Guang Wang acknowledges support from the Australian Research Council under Discovery Project DP180100506.
This work is supported by the National Science Foundation under Grant No.~DMS-1439786 while Zheng Ma and Yu Guang Wang were in residence at the Institute for Computational and Experimental Research in Mathematics in Providence, RI, during Collaborate@ICERM 2019.
Xiaosheng Zhuang acknowledges support by Research Grants Council of Hong Kong (Project No. CityU 11301419).

% references
\section*{References}
\bibliographystyle{apa} % acm, agsm, apa
\bibliography{references}

\end{document}